\documentclass[aoas,preprint,addressasfootnote]{imsart}

\setattribute{journal}{name}{}

\RequirePackage[OT1]{fontenc}
\RequirePackage{amsthm,amsmath}
\RequirePackage{natbib}
\RequirePackage[colorlinks,citecolor=blue,urlcolor=blue]{hyperref}

\usepackage{amssymb}
\usepackage{balance}
\usepackage{array,multirow}
\usepackage{verbatim}
\usepackage{bbm}
\usepackage{xspace}

\usepackage{graphicx,epstopdf} 
\usepackage{epsfig} 
\usepackage{subfigure} 

\usepackage{algorithm}
\usepackage{algorithmic}

\startlocaldefs
\newcommand{\SM}{\textsc{SpectralMirror}\xspace}
\newcommand{\BALD}{\begin{aligned}}
\newcommand{\EALD}{\end{aligned}}
\newcommand{\BALDS}{\begin{aligned*}}
\newcommand{\EALDS}{\end{aligned*}}
\newcommand{\BCAS}{\begin{cases}}
\newcommand{\ECAS}{\end{cases}}
\newcommand{\BEAS}{\begin{eqnarray*}}
\newcommand{\EEAS}{\end{eqnarray*}}
\newcommand{\BEQ}{\begin{equation}}
\newcommand{\EEQ}{\end{equation}}
\newcommand{\BIT}{\begin{itemize}}
\newcommand{\EIT}{\end{itemize}}
\newcommand{\BMAT}{\begin{bmatrix}}
\newcommand{\EMAT}{\end{bmatrix}}
\newcommand{\BNUM}{\begin{enumerate}}
\newcommand{\ENUM}{\end{enumerate}}

\newcommand{\ie}{{i.e.}}

\newcommand{\BA}{\begin{array}}
\newcommand{\EA}{\end{array}}

\newcommand{\ones}{\mathbf 1}

\newcommand{\reals}{\mathbb{R}}

\newcommand{\diag}{\mathop{\mathbf{diag}}}
\DeclareMathOperator{\Expect}{\mathbf{E}}

\DeclareMathOperator{\rank}{rank}

\DeclareMathOperator{\sign}{sgn}

\newcommand{\pc}{\hspace{1pc}}

\newcommand{\abs}[1]{\left| #1 \right|}

\DeclareMathOperator{\linspan}{span}
\newcommand{\norm}[1]{\left\| #1 \right\|}

\newcommand{\iid}{\emph{i.i.d.}}

\newcommand{\note}[1]{\textbf{\textcolor{blue}{#1}}}
\newcommand{\bS}{\mathbf{S}}
\newcommand{\cN}{\mathcal{N}}

\newtheorem{thm}{Theorem}

\newtheorem{lemma}{Lemma}

\def\<{\langle}
\def\>{\rangle}

\def\hU{\widehat{U}}

\newcommand{\prob}{\ensuremath{\mathbf{Pr}}}
\newcommand{\Prob}{\ensuremath{\mathbf{Pr}}}
\newcommand{\expect}{\ensuremath{\mathbb{E}}}
\newcommand{\sgn}{\ensuremath{\mathop{\mathrm{sgn}}}}

\newcommand{\techrep}[2]{#2}
\endlocaldefs

\begin{document}

\begin{frontmatter}

\title{Learning Mixtures of Linear Classifiers}
\runtitle{Learning Mixtures of Linear Classifiers}

\begin{aug}
\author{\fnms{Yuekai} \snm{Sun}\thanksref{m1}\ead[label=e1]{yuekai@stanford.edu}},
\author{\fnms{Stratis} \snm{Ioannidis}\thanksref{m2}\ead[label=e2]{stratis.ioannidis@technicolor.com}}
\and
\author{\fnms{Andrea} \snm{Montanari}\thanksref{m1}\ead[label=e3]{montanari@stanford.edu}}

\runauthor{Lee et al.}

\affiliation{Stanford University\thanksmark{m1}, Technicolor\thanksmark{m2}}

\address{Institute for Computational and Mathematical Engineering \\
Stanford University \\
475 Via Ortega, Stanford, CA 94305 \\
\printead{e1}}

\address{}

\address{Technicolor \\
175 S San Antonio Rd, Los Altos, CA 94022\\
\printead{e2}}

\address{}

\address{Department of Electrical Engineering\\
Department of Statistics \\
Stanford University \\
350 Serra Mall, Stanford, CA 94305  \\
\printead{e3}}

\address{}

\end{aug}

\runauthor{Sun et al.}

\begin{abstract} 
We consider a discriminative learning (regression) problem, whereby the regression function is a convex combination of $k$ linear classifiers. Existing approaches are based on the EM algorithm, or similar techniques, without provable guarantees. We develop a simple method based on spectral techniques and a `mirroring' trick, that discovers the subspace spanned by the classifiers' parameter vectors. Under a probabilistic assumption on the  feature vector distribution, we prove that this approach has nearly optimal statistical efficiency.

\end{abstract} 



\end{frontmatter}

\section{Introduction}\label{sec:intro}
Since Pearson's seminal contribution \citep{pearson1894contributions},
and most notably after the introduction of the 
EM algorithm
\citep{dempster1977maximum}, mixture models and latent variable models
have played a central role in statistics and machine learning, with numerous applications---see, e.g., \citet{mclachlan2004finite}, \citet{bishop1998latent}, and \citet{bartholomew2011latent}. 
Despite their ubiquity, fitting the parameters of a mixture model remains a challenging 
task. The most popular methods (e.g., the EM algorithm or likelihood
maximization by gradient ascent) 
are plagued by local optima
and come with little or no guarantees.
Computationally efficient algorithms with provable guarantees are an
exception in this area. Even 
the idealized problem of learning mixtures of Gaussians has motivated a copious theoretical literature
\citep{sanjeev2001learning,moitra2010settling}.

In this paper we consider the problem of modeling a regression function as a mixture of
$k$ components. Namely, we are given labels $Y_i\in\reals$ and feature vectors $X_{i}\in\reals^d$,
$i\in [n]\equiv \{1,2,\dots,n\}$, and we seek estimates of the parameters of a mixture model
\begin{align}
Y_i\big|_{X_i=x_i}\sim \textstyle \sum_{\ell=1}^kp_{\ell}\,f(y_i| x_i, u_{\ell})\, . \label{eq:GeneralMixtureRegression}
\end{align}
Here $k$ is the number of components, $(p_{\ell})_{\ell \in [k]}$ are weights of the components,
and $u_{\ell}$ is a vector of parameters for the $\ell$-th component. 
Models of this type have been intensely studied in the neural network literature since the 
early nineties \citep{jordan1994hierarchical,bishop1998latent}. They have also found numerous applications ranging from
object recognition \citep{quattoni2004conditional} to machine translation \citep{liang2006end}.
These studies are largely based on learning algorithms without consistency guarantees.

Recently,  \citet{chaganty2013spectral} considered mixtures of linear regressions, 
whereby the relation between labels and feature vectors is linear within each component;
i.e., $Y_i = \<u_{\ell},X_i\>+ {\rm noise}$ (here and below $\<a,b\>$ denotes the standard inner product in 
$\reals^m$). Equivalently, $f(y_i|x_i,u_{\ell}) = f_0(y_i-\<x_i,u_{\ell}\>)$ with $f_0(\,\cdot\,)$ a density of mean zero.
Building on a new approach developed by \citet{hsu2012spectral} and \citet{anandkumar2012tensor},
these authors propose an algorithm for fitting mixtures of linear regressions with provable guarantees. 
The main idea is to regress $Y_i^q$, for $q\in\{1,2,3\}$ against the
tensors $X_i$, $X_i\otimes X_i$, 
$X_i\otimes X_i\otimes X_i$.
The coefficients of these regressions are tensors whose decomposition yields the parameters $u_{\ell}$, $p_{\ell}$.

While the work of \citet{chaganty2013spectral} is a significant step
forward, it leaves several open problems:

\vspace{0.15cm}

\noindent\emph{Statistical efficiency.} Consider a standard scaling of the feature vectors, whereby the
  components $(X_{i,j})_{j\in [p]}$ are of order one. Then, the
  mathematical guarantees of  \citet{chaganty2013spectral} require a
  sample size $n\gg d^6$. This is substantially larger than the
  `information-theoretic' optimal scaling, and is an unrealistic
  requirement in  high-dimension (large $d$).  As noted in \citep{chaganty2013spectral}, this
  scaling is an intrinsic drawback of the tensor approach as it 
  operates in a higher-dimensional  space (tensor space) than
  the space in which data naturally live.

\vspace{0.15cm}

\noindent \emph{Linear regression versus classification.} In virtually all
  applications of  the mixture model
  \eqref{eq:GeneralMixtureRegression}, 
 labels $Y_i$ are categorical---see, e.g., 
\citet{jordan1994hierarchical}, \citet{bishop1998latent}, \citet{quattoni2004conditional}, \citet{liang2006end}. In this case, the
very first step of  Chaganty \& Liang, namely,
regressing $Y_i^2$ on $X_i^{\otimes 2}$ and $Y_i^3$ on
$X_i^{\otimes 3}$, breaks down. Consider---to be definite---the
important case of binary labels (e.g., $Y_i\in\{0,1\}$ or
$Y_i\in\{+1,-1\}$). Then powers of the labels do not provide
additional information (e.g., if $Y_i\in\{0,1\}$, then
$Y_i=Y_i^2$).
Also, since $Y_i$ is non-linearly related to $u_{\ell}$, $Y_i^2$ does not
depend only on $u_{\ell}^{\otimes 2}$.

\vspace{0.15cm}

\noindent \emph{Computational complexity.} The method of
\citet{chaganty2013spectral} solves a regularized linear
regression in $d^3$ dimensions and factorizes a third order tensor  
in $d$ dimensions. Even under optimistic assumptions (finite
convergence of iterative schemes), this requires $O(d^3n+d^4)$
operations.

\vspace{0.15cm}

In this paper, we develop a spectral approach to learning mixtures
of linear classifiers in high dimension. For the sake of simplicity,
we shall focus on the case of binary labels $Y_i\in\{+1,-1\}$, but we
expect our ideas to be more broadly applicable. We consider
regression functions of the form  $f(y_i|x_i,u_{\ell})
=f(y_i|\<x_i,u_{\ell}\>)$, i.e., each component corresponds to a
generalized linear model with parameter vector $u_{\ell}\in\reals^d$.
 In a nutshell, our method  constructs a
symmetric matrix $\hat{Q}\in\reals^{d\times d}$ by taking a suitable
empirical average of the data. The matrix $\hat{Q}$ has the following
property: $(d-k)$ of its eigenvalues are roughly degenerate. The
remaining $k$ eigenvalues correspond to eigenvectors that---approximately---span the same  
subspace as $u_1$, \dots, $u_{k}$. Once this space is
accurately estimated, the problem dimensionality is reduced to $k$;
as such, it is easy to come up with  effective prediction methods (as a matter
of fact, simple $K$-nearest neighbors works very well).

The resulting algorithm is \emph{computationally efficient}, as its most expensive step
is computing the eigenvector decomposition 
of a $d\times d$ matrix (which takes  $O(d^3)$ operations).  
Assuming Gaussian feature vectors $X_i\in\reals^d$, we  prove that our method is also \emph{statistically
  efficient}, i.e., it only requires $n \ge d$ samples to
accurately reconstruct the subspace spanned by $u_1,\dots,u_k$.
This is the same amount of data needed to estimate the covariance of
the feature vectors $X_i$ or a parameter vector $u_1\in\reals^d$ in the trivial case
of a mixture with a single component, $k=1$. It is unlikely that a
significantly better efficiency can be achieved without additional structure.

The assumption of Gaussian feature vectors $X_i$'s is admittedly
restrictive.
On one hand, as for the problem of learning mixtures of
Gaussians \citep{sanjeev2001learning,moitra2010settling}, we believe
that useful insights can be gained by studying this simple setting. On the
other,  and as discussed below, our proof does not really require the
distribution of the $X_i$'s to be Gaussian, and a strictly weaker
assumption is sufficient. We expect that future work will succeed in
further relaxing this assumption.
%
%
\subsection{Technical contribution and related work}

Our approach is related to the principal Hessian directions (pHd)
method
proposed by \citet{li1992principal} and further developed by \citet{cook1998principal}
and co-workers. PHd is an approach to dimensionality reduction and data 
visualization. It generalizes principal component analysis to the
regression (discriminative) setting, whereby each data point consists
of a feature vector $X_i\in\reals^d$ and a label $Y_i\in\reals$.
Summarizing, the idea is to form the `Hessian' matrix $\hat{H} =
n^{-1}\sum_{i=1}^nY_i\, X_iX_i^T\in\reals^{d\times d}$. (We assume here, for ease of
exposition, that the $X_i$'s have zero mean and unit covariance.) 
The eigenvectors associated to eigenvalues with largest magnitude
are used to identify a subspace in $\reals^d$ onto which to project
the feature vectors $X_i$'s.

Unfortunately, the pHd approach fails in general for the mixture models of interest here, namely, mixtures of linear classifiers. For instance, it fails when each component of \eqref{eq:GeneralMixtureRegression} is described by a logistic
model $f(y_i=+1| z) = (1+e^{-z})^{-1}$,  when features are
centered at ${\mathbb E}(X_i)=0$; a proof can be found in \techrep{the extended version of this paper~\citep{full_version}.}{Appendix~\ref{app:phd}.}

Our approach overcomes this problem by constructing $\hat{Q} =
n^{-1}\sum_{i=1}^nZ_i\, X_iX_i^T\in\reals^{d\times d}$. The $Z_i$'s
are pseudo-labels obtained by applying a  `mirroring' transformation 
to the $Y_i$'s. Unlike with $\hat{H}$, the eigenvector structure of
$\hat{Q}$ enables us to estimate the span of $u_1,\dots,u_k$.

As an additional technical contribution, we establish non-asymptotic 
bounds on the estimation error that allow to characterize
the trade-off between the data dimension $d$ and the sample size
$n$.  In contrast, rigorous analysis on pHd is limited to the low-dimensional
regime of $d$ fixed as $n\to\infty$. It would be interesting to
generalize the analysis developed here to characterize the
high-dimensional properties of pHd as well.

\section{Problem Formulation}\label{sec:problem}
\begin{figure*}[!t]
\hspace*{\stretch{1}}\includegraphics[trim=2.1cm 0 0 0 ,clip,width=0.32\textwidth]{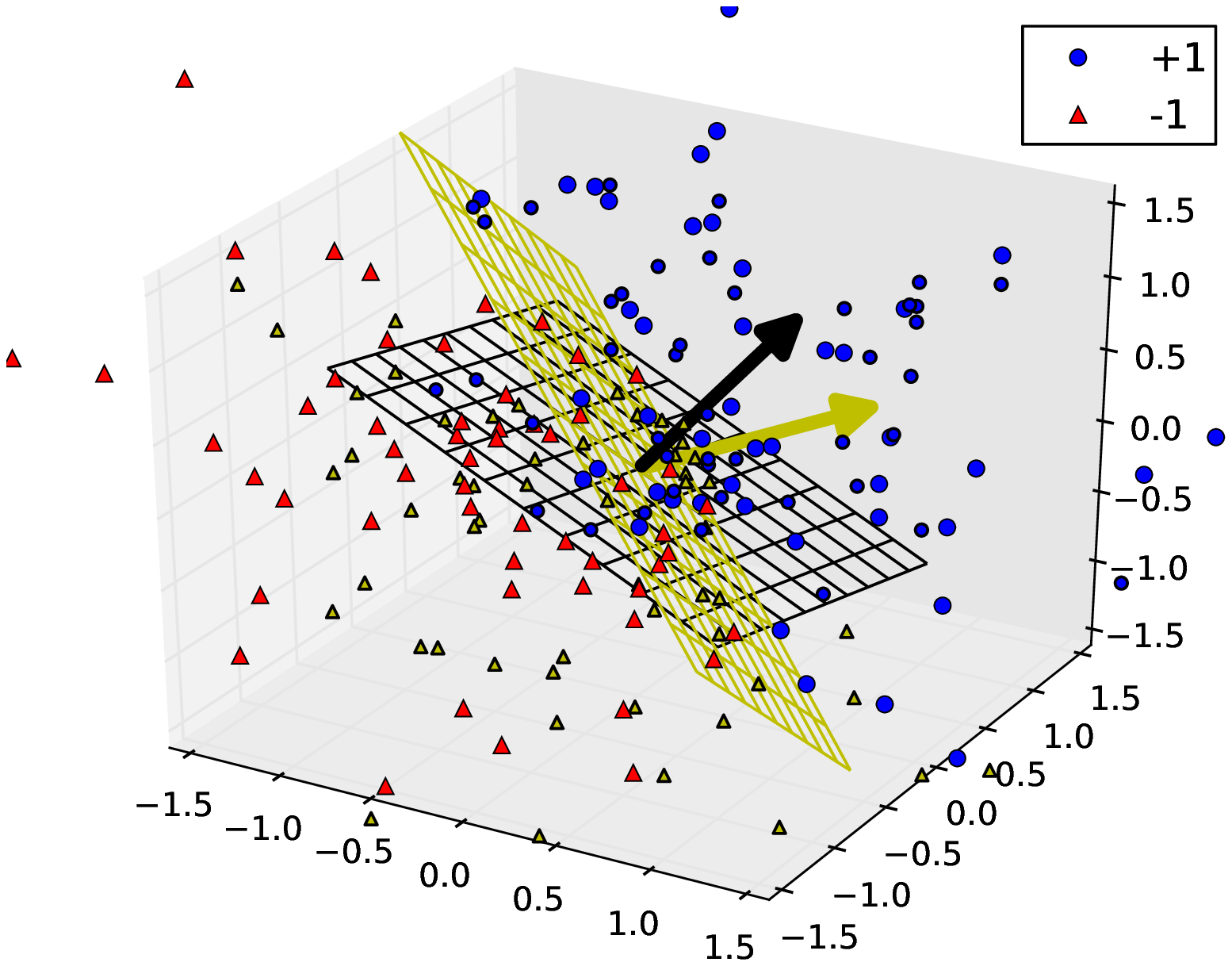}
\hspace*{\stretch{2}}\includegraphics[trim=2.1cm 0 0 0 ,clip,width=0.32\textwidth]{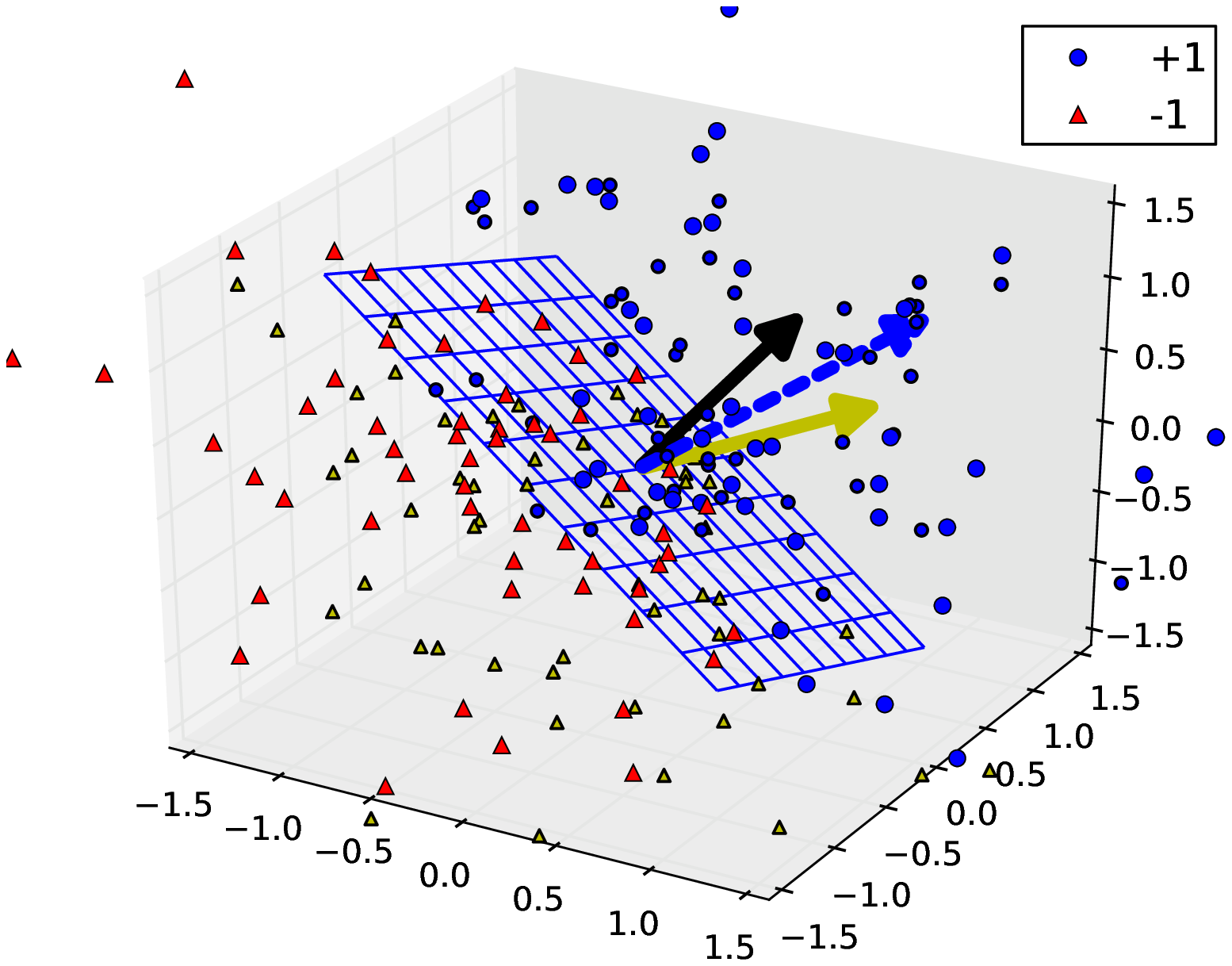}
\hspace*{\stretch{2}}\includegraphics[trim=2.1cm 0 0 0 ,clip,width=0.32\textwidth]{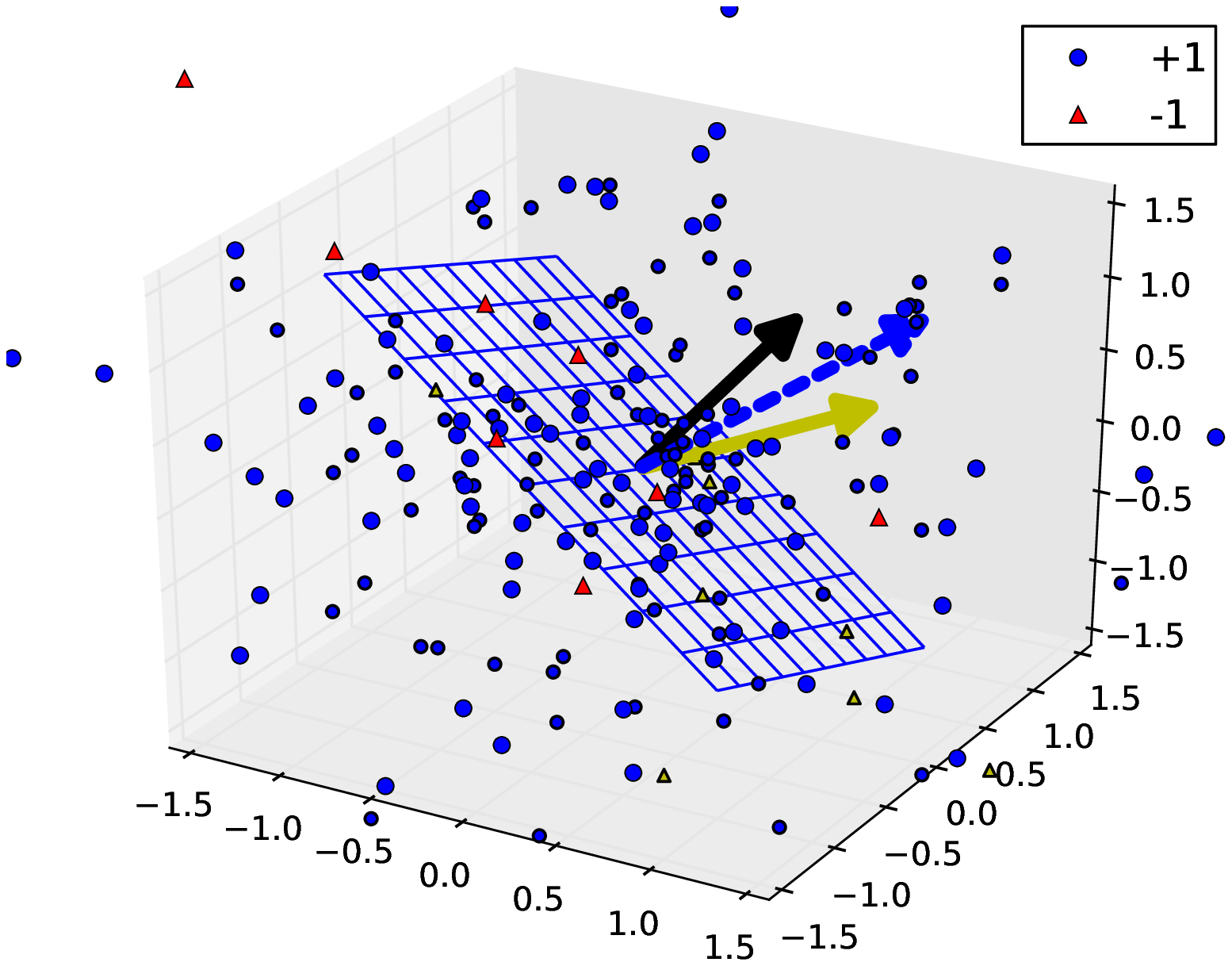}
\hspace*{\stretch{1}}\\
\hspace*{\stretch{1}}(a)
\hspace*{\stretch{2}}(b)
\hspace*{\stretch{2}}(c)\hspace*{\stretch{1}}
\caption{The mirroring process applied to a mixture of two 3-dimensional classifiers. Figure (a) shows labels generated by two classifiers in $\reals^3$; the figure includes the parameter profiles as well as the corresponding classification surfaces. Figure (b) shows  the mirroring direction $\hat{r}$ as a dashed vector, computed by \eqref{eq:mirror}, as well as the plane it defines; note that $\hat{r}$ lies within the positive cone spanned by the two classifier profiles, approximately. Finally, Figure (c) shows the result of the mirroring process: the region of points that was predominantly positive has remained unaltered, while the region of points that was predominantly negative has been flipped.  }\label{figure:3dplots}
\end{figure*}

\subsection{Model}
Consider a dataset comprising $n$ i.i.d.~pairs $(X_i,Y_i)\in \reals^d\times \{-1,+1\}$, $i\in [n]$. We refer to the vectors $X_i\in \reals^d$ as \emph{features} and to the binary variables as \emph{labels}.
We assume that the features $X_i\in\reals^d$ are sampled from a Gaussian distribution with mean $\mu\in \reals^d$ and a positive definite covariance $
\Sigma\in \reals^{d\times d}$.   The labels $Y_i\in \{-1,+1\}$  are generated by a \emph{mixture of linear classifiers}, i.e., 
\begin{align}\label{eq:mix}
\prob(Y_i=+1\mid X_i) = 
\textstyle \sum_{\ell=1}^k p_\ell\, f(\<u_\ell,X_i\>) \, .
\end{align}
Here, $k\ge 2$ is the number of components in the mixture; $(p_{\ell})_{\ell\in [k]}$ are the weights,
satisfying of course $p_{\ell}>0$, $\sum_{\ell=1}^kp_{\ell} = 1$;  and $(u_\ell)_{\ell\in [k]}$, 
$u_{\ell}\in\reals^d$ are the normals to the planes defining  the $k$ linear classifiers. We refer to each normal $u_\ell$ as the \emph{parameter profile} of the $\ell$-th classifier; we assume that the  profiles $u_\ell$, $\ell\in[k]$, are linearly independent, and that $k<n/2$.  

We assume that the function $f:\reals\to [0,1]$,  characterizing
the classifier response, 
is analytic, non-decreasing, strictly concave in $[0,+\infty)$, and satisfies: \begin{align}\label{fprops}\lim_{t\to\infty}\!f(t)\!=\!1, ~~\lim_{t\to -\infty}\!f(t)\!=\!0, \quad 1\!-\!f(t)\!=\!f(-t).\end{align}
 As an example, it is useful to keep in mind the logistic function $f(t) = (1+e^{-t})^{-1}$. Fig.~\ref{figure:3dplots}(a) illustrates a mixture of $k=2$ classifiers over $d=3$ dimensions.

\subsection{Subspace Estimation, Prediction and Clustering}

Our main focus is the following task:

\begin{quote}\textbf{Subspace Estimation:}  After observing $(X_i,Y_i),$ $i\in [n]$, estimate the subspace spanned by the profiles of the $k$ classifiers, \ie, $U \equiv {\rm span}(u_1,\dots,u_k)$.
\end{quote}

For $\hU$ an estimate of $U$, we characterize performance via the \emph{principal angle} between the two spaces, namely
$$d_P(U,\hU) = \max_{ x\in U, y \in \hU} \arccos\left(\tfrac{\<x,y\>}{\|x\|\|y\|}\right).$$
Notice that projecting the features $X_i$ on $U$ entails no loss of information w.r.t.~\eqref{eq:mix}.  This can be exploited  to improve the performance of several learning tasks through dimensionality reduction, by projecting the features to the estimate of the subspace $U$. Two such tasks are:

\begin{quote}\textbf{Prediction}: Given a new feature vector $X_{n+1} $, predict the corresponding 
label $Y_{n+1}$.
\end{quote} 

\begin{quote}\textbf{Clustering}: Given a new feature vector and label pair $(X_{n+1},Y_{n+1})$, identify the classifier that generated the label. 
\end{quote}

As we will see in Section~\ref{sec:experiments}, our subspace estimate can be used to significantly improve the performance of both prediction and clustering.
\subsection{Technical Preliminary}
We review here a few definitions used in our exposition. 
The \emph{sub-gaussian norm} of a random variable $X$ is:
$$
\|X\|_{\psi_2} =\sup_{p \ge 1}\frac{1}{\sqrt{p}}(\expect[|X|^p])^{1/p}.
$$
We say $X$ is sub-gaussian if $\|X\|_{\psi_2}<\infty$.
We say that a random vector $X\in \reals^d$ is sub-gaussian if $\<y,X\>$ is sub-gaussian for any $y$ on the unit sphere $\bS^{d-1}$. 

We  use the following variant of Stein's identity \cite{stein1973estimation,liu1994Siegel}. Let  $X\in \reals^d$, $X'\in \reals^{d'}$ be jointly Gaussian random vectors, and consider a function \mbox{$h:\reals^{d'}\to\reals$} that is almost everywhere (a.e.)~differentiable and satisfies $\expect[|\partial h(X')/\partial x_i|]<\infty$, $i\in [d']$. Then, the following identity holds:
\begin{align}\mathrm{Cov}(X,h(X')) = \mathrm{Cov}(X,X')\expect[\nabla h(X') ]. \label{stein}\end{align}

\section{Subspace Estimation}\label{sec:algorithm}
\begin{algorithm}[!t]
  \caption{\SM}\label{algo:spectral_mirror}
  \begin{algorithmic}[1]
    \REQUIRE{Pairs $(X_i,Y_i)$, $i \in [n]$}
    \ENSURE{Subspace estimate $\hU$}
    \STATE{$ \hat{\mu} \leftarrow \frac{1}{\lfloor n/2\rfloor} \sum_{i=1}^{\lfloor n/2\rfloor} X_i$ \label{line:mu}}
\STATE{$\hat{\Sigma} \leftarrow \frac{1}{\lceil n/2\rceil} \sum_{ i = 1}^{\lfloor n/2\rfloor} (X_i - \hat{\mu})(X_i - \hat{\mu})^T$ \label{line:sigma}}
\STATE{$\hat{r} \leftarrow  \frac{1}{\lfloor n/2 \rfloor} \sum_{i=1}^{\lfloor n/2\rfloor} Y_i  \hat{\Sigma}^{-1}(X_i\!-\!\hat{\mu})$ \label{line:direction}}
\STATE{ \textbf{for each} $i \in\{ \lfloor n/2\rfloor+1,\dots,n\}$: \vspace*{-0.2cm} \label{line:mirror}
 $$  Z_i \leftarrow Y_i \sgn\<\hat{r},X_i\> \vspace*{-0.6cm}$$}
\STATE{ $\displaystyle \hat{Q}\! \!\leftarrow\!\! \frac{1}{\lceil n/2\rceil }\!\!\sum_{i=\lfloor n/2\rfloor+1}^n\!\!\!\!\! Z_i \hat{\Sigma}^{-1/2}(X_i\!-\!\hat{\mu})(X_i\!-\!\hat{\mu})^T\hat{\Sigma}^{-1/2}$ \label{line:Q}}
\STATE{Find eigendecomposition $\sum_{\ell=1}^d \lambda_\ell w_\ell w_\ell^T
 $ of $\hat{Q}$ \label{line:eig}}
\STATE{Let $\lambda_{(1)}, \ldots, \lambda_{(k)}$ be the $k$ eigenvalues furthest from the median. \label{line:exteig} }
 \STATE{$\hat{U} \leftarrow \linspan\left(\hat{\Sigma}^{-1/2} w_{(1)},\ldots,\hat{\Sigma}^{-1/2} w_{(k)}\right)$\label{line:hatu}}
  \end{algorithmic}
\end{algorithm}

In this section, we present our algorithm for subspace estimation, which we refer to as \SM. Our main technical contribution, stated formally below, is that the output $\hat{U}$ of \SM is a consistent estimator of the subspace $U$ 
as soon as {$n\ge C\,d$}, for a sufficiently large constant $C$.

\subsection{Spectral Mirror  Algorithm}

We begin by presenting our algorithm for estimating the subspace span $U$. Our algorithm consists of three main steps. First, as pre-processing, we estimate the mean and covariance of the underlying features $X_i$. Second, using these estimates, we identify a vector $\hat{r}$ that concentrates near the convex cone spanned by the profiles $(u_\ell)_{\ell \in [k]}$. We use this vector to perform an operation we call \emph{mirroring}: we `flip' all labels lying in the negative halfspace determined by $\hat{r}$. Finally, we compute a weighted covariance matrix $\hat{Q}$ over all $X_i$, where each point's contribution is weighed by the mirrored labels: the eigenvectors of this matrix, appropriately transformed, yield the span $U$.

 These operations are summarized in Algorithm~\ref{algo:spectral_mirror}. We discuss each of the main steps in more detail below:

\noindent\textbf{Pre-processing.} (Lines 1--2)  We split the dataset into two halves. Using the first half (\emph{i.e.}, all $X_i$ with $1\leq i \leq \lfloor \frac{n}{2} \rfloor)$, we construct estimates $\hat{\mu}\in\reals^d$ and $\hat{\Sigma}\in \reals^{d\times d}$ of the feature mean and covariance, respectively.  Standard Gaussian (\ie, `whitened') versions of features $X_i$ can be constructed as $\hat{\Sigma}^{-1/2}(X_i\!-\!\hat{\mu})$.

\noindent{\textbf{Mirroring.}} (Lines 3--4)  We compute the vector:
\begin{align}\hat{r} =  \frac{1}{\lfloor n/2 \rfloor} \textstyle\sum_{i=1}^{\lfloor n/2\rfloor} Y_i \hat{\Sigma}^{-1}(X_i\!-\!\hat{\mu})  \in \reals^d. \label{eq:mirror} \end{align}
We refer to $\hat{r}$ as the \emph{mirroring direction}. In Section~\ref{sec:proof}, we show that $\hat{r}$
concentrates around 
its population ($n=\infty$) version $r\equiv \expect[ Y\Sigma^{-1}(X-\mu)]$. Crucially,  $r$  lies in the interior of the convex cone spanned by the parameter profiles, i.e., $r=\sum_{\ell=1}^k \alpha_\ell u_\ell$, for some positive $\alpha_\ell>0$, $\ell\in [k]$ (see Lemma~\ref{lemma:quad} and Fig.~\ref{figure:3dplots}(b)). 
Using this $\hat{r}$, we `mirror' the labels in the second part of the dataset:
$$Z_i = Y_i \sgn\<\hat{r},X_i\>, \quad \text{for~} \lfloor n/2\rfloor< i \leq n. $$
In words, $Z_i$ equals $Y_i$ for all $i$ in the positive half-space defined by the mirroring direction; instead, all labels for points $i$ in the negative half-space are flipped (\ie, $Z_i=-Y_i$). This is illustrated in Figure~\ref{figure:3dplots}(c). 

\noindent{\textbf{Spectral Decomposition.}} (Lines 5--8) 
The mirrored labels are used to compute a weighted covariance matrix over  whitened features as follows:
\begin{align*}
\hat{Q} & =\frac{1}{\lceil \frac{n}{2}\rceil} \sum_{i=\lfloor n/2\rfloor+1}^n Z_i\hat{\Sigma}^{-1/2}(X_i-\hat{\mu})(X_i-\hat{\mu})^T \hat{\Sigma}^{-1/2}
\end{align*}
  The spectrum of $\hat{Q}$ has a specific structure, that reveals the span $U$. In particular, as we will see in Section~\ref{sec:proof}, $\hat{Q}$ converges to a matrix $Q$ that contains an eigenvalue with multiplicity $n-k$; crucially, the eigenvectors corresponding to the remaining $k$ eigenvalues, subject to the linear transform $\hat{\Sigma}^{-1/2}$, span the subspace $U$.  As such, the final steps of the algorithm amount to discovering the eigenvalues that `stand out' (\ie, are different from the eigenvalue with multiplicity $n-k$), and rotating the corresponding eigenvectors to obtain $\hat{U}$. More specifically, let $(\lambda_\ell,w_\ell)_{\ell\in [d]}$ be the eigenvalues and eigenvectors of $\hat{Q}$. Recall that $k< n/2$. The algorithm computes the median of all eigenvalues, and identifies the $k$ eigenvalues furthest from this median; these are the `outliers'.  The corresponding $k$ eigenvectors, multiplied by $\hat{\Sigma}^{-1/2}$, yield the subspace estimate $\hat{U}$. 

The algorithm \emph{does not require} knowledge of the classifier response function $f$. 
Also, while we assume knowledge of  $k$, an eigenvalue/eigenvectors  statistic (see, e.g., \citet{zelnik2004self}) can  be used to estimate $k$, as the number of `outlier' eigenvalues. 

\subsection{Main Result}
Our main result states that \SM is a consistent estimator of the subspace spanned by $(u_\ell)_{\ell\in [k]}$.
This is true for `most'  $\mu\in \reals^d$.  Formally, we say that an event occurs for \emph{generic} $\mu$  
if adding an arbitrarily small random perturbation to $\mu$, the event occurs with probability $1$ w.r.t.~this perturbation. 
\begin{thm} \label{thm:main}
Denote by $\hat{U}$ the output of \SM, and let $P^\bot_{r} \equiv I -r r^T/\|r\|^2$ be 
the projector orthogonal to $r$, given by \eqref{eq:RDefinition}. Then, for generic  $\mu$, as well as for $\mu=0$, there exists $\epsilon_0>0$ such that, 
for all  $\epsilon\in [0,\epsilon_0)$,
\begin{align*}
\Prob(d_P(P^\bot_{r} U,\hat{U}) > \epsilon) 
\le C_1\exp(-C_2 \frac{n\epsilon^2}{d}).
\end{align*}
Here $C_1$ is an absolute constant, and $C_2>0$ depends on $\mu$, $\Sigma$, $f$ and $(u_{\ell})_{\ell\in [k]}$.
\end{thm}
In other words, $\hat{U}$ provides an accurate estimate of $P^\bot_{r} U$ as soon as $n$ is significantly larger than
$d$. This holds for generic $\mu$, but we also prove that it holds for the specific and important case where $\mu=0$; in fact, it also holds for all small-enough $\mu$. Note that this does not guarantee that $\hat{U}$ spans the direction $r\in U$; nevertheless, as shown below, the latter is accurately estimated by $\hat{r}$ (see Lemma~\ref{lemma:hatr_convergence}) and can be added to the span, if necessary. Moreover, our experiments suggest this is rarely the case in practice, as $\hat{U}$ indeed includes the direction $r$ (see Section~\ref{sec:experiments}).

\section{Proof of Theorem~\ref{thm:main}}\label{sec:proof}

 Recall that we denote by $r$
the population ($n=\infty$) version  of $\hat{r}$.
Let  $g(s) \equiv 2f(s)-1$, for $s\in \reals$, and
 observe that
$\expect[Y\mid X=x] = \sum_{\ell =1}^kp_\ell g(\<u_\ell,x\>).$ 
Hence,
\begin{align}
 r = \expect\left[\Sigma^{-1}(X-\mu) \cdot \left(\textstyle\sum_{\ell =1}^kp_\ell g(\<u_\ell,X\>)\right) \right].
\label{eq:RDefinition}
\end{align}
%
%
Then, the following concentration result holds:
\begin{lemma} \label{lemma:hatr_convergence}
There exist an absolute constant $C>0$ and $c_1,c_1',c_2'$ that depend on  $\|X\|_{\psi_2}$ such that:
$$
\Prob(\|\hat{r} - r\|_2 \!\ge\! \epsilon)\!\le\! Ce^{-\min\bigl\{\frac{c_2n\epsilon^2}{d},\bigl(c_1'\sqrt{n}\epsilon - c_2'\sqrt{d}\bigr)^2\bigr\}}.
$$
\end{lemma}
The proof of Lemma~\ref{lemma:hatr_convergence} relies on a large deviation inequality for sub-Gaussian vectors, and is provided in \techrep{\citep{full_version}.}{Appendix~\ref{app:hatr}.} Crucially, $r$ 
lies in the interior of the convex cone spanned by the parameter profiles:
\begin{lemma}\label{lemma:quad}
$r = \sum_{\ell=1}^k \alpha_\ell u_\ell$ for some $\alpha_\ell> 0$, $\ell\in [k]$. 
\end{lemma}
\begin{proof}
From  \eqref{eq:RDefinition},
$$\textstyle r = \sum_{\ell=1}^k p_\ell \Sigma^{-1}\expect[(X-\mu)g(\<u_\ell,X\>)]. $$
It thus suffices to show that $\Sigma^{-1}\expect[(X-\mu)g(\<u,X\>)]=\alpha u$, for some $\alpha>0$.
Note that $X'=\langle u, X\rangle$ is normal with mean $\mu_0=u^T\mu$ and variance $\sigma_0^2=u^T\Sigma u>0$. 
Since $f$ is analytic and non-decreasing, so is $g$; moreover,  $g'\geq 0$. 
This, and the fact that $g$ is non-constant, implies $\expect[g'(X')]>0$. On the other hand, from Stein's identity~\eqref{stein}, $\expect[g'(X')]= \frac{1}{\sigma_0^2} \expect[X'g(X')]<\infty$, as $g$ is bounded. Hence:
\begin{align*}
&\Sigma^{-1}\expect[(X-\mu)g(\<u,X\>)]\\
& \stackrel{\eqref{stein}}{=} \Sigma^{-1}\mathrm{Cov}(X,\langle u, X\rangle) \expect[g'( X')],~\text{where}~X'\sim\mathcal{N}(\mu_0,\sigma_0^2)\\
& = \Sigma^{-1}\cdot \expect[(X-\mu)X^Tu] \cdot  \expect[g'( X')]\\
& = \Sigma^{-1}\cdot \Sigma u\cdot  \expect[g'( X')] =  \expect[g'( X')] \cdot u
\end{align*}
and the lemma follows. 
\end{proof}

For $r$ and $(\alpha_\ell)_{\ell\in [k]}$ as in Lemma~\ref{lemma:quad}, define
\begin{align*}
z(x) &= \expect[ Y\sgn(\langle r,X \rangle )\mid X=x] \\
&=\textstyle \left(\sum_{\ell=1}^k p_\ell g(\<x,u_\ell\>)\right) \cdot \sgn\left(\sum_{\ell=1}^k\alpha_\ell \<x,u_\ell\>\right).   
\end{align*} 
Observe that $z(x)$ is the expectation of the mirrored label at a point $x$ \emph{presuming that the mirroring direction is exactly $r$}. 
Let $Q\in \reals^{d\times d}$ be the matrix: 
$$
Q=\expect[z(X)\Sigma^{-1/2}(X-\mu)(X-\mu)^T\Sigma^{-1/2}].
$$
Then $\hat{Q}$ concentrates around  $Q$, as stated below. 
\begin{lemma} 
\label{lem:q-large-deviation}
Let $\epsilon_0 \equiv \min\{\alpha_1,\dots,\alpha_k\}\sigma_{\min}(U)$, where the $\alpha_\ell>0$
are defined as per Lemma \ref{lemma:quad} and $\sigma_{\min}(U)$ is the smallest non-zero singular value of $U$.  Then 
for $\epsilon<\min(\epsilon_0,\|r\|/2)$: 
$$
\Prob(\|\hat{Q} - Q\|_2 > \epsilon) \le C\exp\{-F(\epsilon^2)\},
$$
where $
F(\epsilon) \equiv \min\left\{\frac{c_1n\epsilon^2}{d},\bigl(c_1'\sqrt{n}\epsilon - c_2'\sqrt{d}\bigr)^2\right\},
$ 
 $C$ an absolute constant, and $c_1,\,c_1',\,c_2'$
 depend on $\mu$, $\Sigma$, and $\|r\|$.
\end{lemma}
The proof of Lemma~\ref{lem:q-large-deviation} is \techrep{also provided in \citep{full_version}.}{in Appendix~\ref{app:lemmQ}.} We again rely on large deviation bounds for sub-gaussian random variables; nevertheless, our proof diverges from standard arguments because $\hat{r}$, rather than $r$, is used as  a mirroring direction. Additional care is needed to ensure that (a) when $\hat{r}$ is close enough to $r$, its projection to $U$ lies in the interior of the convex cone spanned by the profiles, and (b) although $\hat{r}$ may have a (vanishing) component outside the convex cone, the effect this has on $\hat{Q}$ is negligible,
for $n$ large enough.

An immediate consequence of Lemma~\ref{lemma:quad} is that $r$ reveals a direction in the span $U$.  The following lemma states that the eigenvectors of $Q$, subject to a rotation, yield the remaining $k-1$ directions: 
\begin{lemma}\label{lemma:spectrum}
Matrix $Q$ has at most $k+1$ distinct eigenvalues. One eigenvalue, termed $\lambda_0$, has multiplicity $d-k$. For generic $\mu$, as well as for $\mu=0$, the eigenvectors $w_{1},\ldots,w_{k}$ corresponding to the remaining eigenvalues $\lambda_{1}$, \ldots, $\lambda_{k}$ are such that
$$P^\bot_{r} U = \linspan(P^\bot_{r}\Sigma^{-1/2}w_{1},\ldots,P^\bot_{r}\Sigma^{-1/2}w_{k}), $$
where $P^\bot_r$ is the projection orthogonal to $r$.
\end{lemma}
\begin{proof}
Note that
\begin{align*}
Q&=\expect[z(X)\Sigma^{-\frac{1}{2}}(X-\mu)(X-\mu)^T\Sigma^{-\frac{1}{2}}]\nonumber\\ 
&=\expect[z(\Sigma^{1/2}W+\mu)WW^T],\quad\text{where } W\sim \mathcal{N}(0,I)\nonumber\\
& =\expect\big[\sum_{\ell=1}^k p_\ell g(\<\Sigma^{\frac{1}{2}} W\!+\!\mu,\!u_\ell\>)  \sgn(\<\Sigma^{\frac{1}{2}} W\!+\!\mu,\! r\>)WW^T\big]\\
& =\expect\big[\sum_{\ell=1}^k p_\ell g(\< W+\tilde{\mu},\tilde{u}_\ell\>)\sgn(\< W+\tilde{\mu},\tilde{ r}\> )WW^T\big]
\end{align*}
for $\tilde{u}_\ell\equiv \Sigma^{\frac{1}{2}}u_\ell$, $\tilde{r} \equiv \Sigma^{\frac12}r$, and $\tilde{\mu} \equiv \Sigma^{-\frac12}\mu$.
Hence $Q = \sum_{\ell=1}^k p_\ell Q_\ell$ where
$$Q_\ell = \expect[g(\<\tilde{u}_\ell,W+\tilde{\mu}\>)\sgn(\<\tilde{r},W+\tilde{\mu}\>)WW^T].$$
By a  rotation invariance argument, $Q_\ell$ can be written as
\begin{align}Q_\ell = a_\ell I + b_\ell(\tilde{u}_\ell \tilde{r}^T + \tilde{r}\tilde{u}_\ell^T) +c_\ell \tilde{u}_\ell\tilde{u}_\ell^T+d_\ell \tilde{r}\tilde{r}^T \label{eq:breakdown}\end{align}
for some $a_\ell,b_\ell,c_\ell, d_\ell \in \reals$. 
To see this,  let $\tilde{Q}_\ell=[\tilde{q}_{ij}]_{i,j\in [d]}$, and suppose first that 
\begin{align}\tilde{r} = [\tilde{r}_1,\tilde{r}_2, 0,\ldots, 0]\text{ and }\tilde{u}_\ell = [\tilde{u}_{\ell1}, \tilde{u}_{\ell2}, 0,\ldots, 0].\label{allzeros}\end{align}
  Since $W$ is whitened,  its coordinates are independent. Thus, under \eqref{allzeros}, $\tilde{q}_{ij}=0$ for all $i\neq j$ s.t.~$i,j>2,$ and $\tilde{q}_{ii}=a_\ell$ for $i>2$, for some $a_\ell$. Thus $\tilde{Q}_\ell=a_\ell I + B$, where $B$ is symmetric and  0 everywhere except perhaps on $B_{11},B_{12},B_{21},B_{22}$ (the top left block). Since the profiles $u_\ell$ are linearly independent, so are $\tilde{u}_\ell$ and $\tilde{r}$, by Lemma~\ref{lemma:quad}. Hence, matrices $\tilde{u}_\ell\tilde{r}^T+\tilde{r}\tilde{u}_\ell^T, \tilde{u}_\ell\tilde{u}_\ell^T, \tilde{r}\tilde{r}^T$ span all such $B$, so \eqref{eq:breakdown} follows.  
Moreover, since $W$ is whitened, $\tilde{Q}_\ell$ is rotation invariant and thus \eqref{eq:breakdown} extends beyond \eqref{allzeros}; indeed, if $\tilde{r}'=R\tilde{r}$, $\tilde{u}'_\ell=R\tilde{u}_\ell$, $\tilde{\mu}'=R\tilde{\mu}$ where $R$ a rotation matrix (i.e. $RR^T=I$), then $Q'=RQR^T$. Hence, as \eqref{allzeros} holds for some orthonormal basis,  \eqref{eq:breakdown} holds for all bases.

Let $a= \sum_{\ell=1}^k p_\ell a_\ell$. Then
\begin{align*}
Q - a I =&  \sum_{\ell=1}^k p_\ell d_\ell \tilde{r}\tilde{r}^T +\tilde{r}(\sum_{\ell =1 }^k p_\ell b_\ell \tilde{u}_\ell)^T +\\
&+(\sum_{\ell =1 }^kp_\ell b_\ell\tilde{u}_\ell)\tilde{r}^T+\sum_{\ell=1}^k p_\ell c_\ell \tilde{u}_\ell \tilde{u}_\ell^T.
\end{align*}
Let $P^\bot_{\tilde{r}}$ be the projector orthogonal to $\tilde{r}$, i.e., $P^\bot_{\tilde{r}} = I -\frac{\tilde{r}{\tilde{r}}^T}{\|\tilde{r}\|_2^2}. $
Let $v_\ell \equiv P^\bot_{\tilde{r}} \tilde{u}_\ell$.
 Lemma~\ref{lemma:quad} and the linear independence of $\tilde{u}_\ell$ imply that $v_\ell\neq 0$, for all $\ell\in [k]$.
Define $R\equiv P^\bot_{\tilde{r}} (Q-a I) P^\bot_{\tilde{r}} = \textstyle\sum_{\ell=1}^k \gamma_\ell v_\ell v_\ell^T, $
where $\gamma_\ell = p_\ell c_\ell$, $\ell \in [k]$.
We will show  below that for generic $\mu$, as well as for $\mu=0$,  $\gamma_\ell\neq 0$ for all $\ell\in[k]$.
This implies that  $\rank(R) = k-1$. Indeed, $R =P^\bot_{\tilde{r}} \sum \gamma_\ell \tilde{u}_\ell\tilde{u}_\ell^T P^\bot_{\tilde{r}}= P^\bot_{\tilde{r}}\tilde{R} P^\bot_{\tilde{r}},$ where $\tilde{R}$ has rank $k$ by the linear independence of profiles. As $P_\bot$ is a projector orthogonal to a 1-dimensional space, $R$ has rank at least $k-1$.  On the other hand, $\mathrm{range}(R) \subseteq \tilde{U}$, for $\tilde{U} =  \linspan( \tilde{u}_1,\ldots,\tilde{u}_\ell)$, and $\tilde{r}^TR\tilde{r}=0$ where $\tilde{r}\in \tilde{U}\setminus \{0\}),$ so $\rank(R)=k-1$. The latter also implies that $\mathrm{range}(R) = P^\bot_{\tilde{r}} \tilde{U}$, as $\mathrm{range}(R) \bot \tilde{r}$, $\mathrm{range}(R)\subseteq \tilde{U}$, and $\dim(\mathrm{range}(R))=k-1$.

The above imply that $Q$ has one eigenvalue of multiplicity $n-k$, namely $a$. Moreover, the eigenvectors $w_{1},\ldots,w_{k}$ corresponding to the remaining eigenvalues (or, the non-zero eigenvalues of $Q-aI$)  are such that $$P^\bot_{\tilde{r}}\Sigma^{1/2} U=P^\bot_{\tilde{r}}\linspan(w_{1},\ldots,w_{k}).$$
The lemma thus follows by multiplying both sides of the above equality with $P^\bot_{r}\Sigma^{-1/2}$, and using the fact that $P^\bot_{r}\Sigma^{-1/2}P^\bot_{\tilde{r}} = P^\bot_r\Sigma^{-1/2}.$

It remains to show that $\gamma_\ell\neq 0$, for all $\ell \in [k]$, when $\mu$ is generic or 0. 
 Note that
\begin{align}
c_\ell &\langle \tilde{u}_\ell, v_\ell\rangle^2 \stackrel{\eqref{eq:breakdown}}{=} \langle v_\ell, (Q_\ell -a_\ell I) v_\ell \rangle = \label{tildec}
\\ &\mathrm{Cov}(g(\langle\tilde{u}_\ell,W+\tilde{\mu} \rangle)  \sgn(\langle \tilde{r}, W +\tilde{\mu} \rangle) ; \langle W , v_\ell\rangle^2)\equiv \tilde{c}_\ell \nonumber
\end{align}
It thus suffices to show that $\tilde{c}_\ell\neq 0$. 
Lemma~\ref{lemma:quad} implies that $\tilde{u}_\ell = v_\ell+c \tilde{r}$ for some $c>0$, 
hence
\begin{align*}
\tilde{c}_\ell = \mathrm{Cov}[ g(X + c Y+z_\ell(\mu) \rangle)\sgn(Y+z_0(\mu)) ; X^2 ],
\end{align*}
where $X \equiv \langle v_\ell, W \rangle $ and $Y \equiv \langle \tilde{r},W\rangle $ are independent Gaussians with  mean 0, and $z_\ell(\mu) \equiv \<\tilde{u}_\ell,\tilde{\mu}\>$, $z_0(\mu)\equiv\<\tilde{r},\tilde{\mu}\>$. Hence,  $\tilde{c}_\ell = \mathrm{Cov}[F(X);X^2]$ where 
\begin{align*}
&F(x) = \expect_Y[g(x+cY+z_\ell(\mu))\sgn(Y+z_0(\mu))]\\
&= \int_{-z_0(\mu)}^\infty\!\!\!\!\!\!\!\!\!\!\!\!\!\!\! g(x\!+\!cy+\!z_\ell(\mu))\phi(y)dy \!- \!\!\int_{-\infty}^{-z_0(\mu)}\!\!\!\!\!\!\!\!\!\!\!\!\!\!\!\!\!g(x+cy+z_\ell(\mu)\phi(y)dy
\end{align*}
where $\phi$ the normal p.d.f.
Assume first that $\mu = 0$.  By~\eqref{fprops},  $g$ is anti-symmetric, i.e., $g(-x) = -g(x)$. Thus, $F(-x)=\expect_Y[g(-x+cY)\sgn(Y)]  
\stackrel{Y'\equiv -Y}{=} \expect_{Y'}[g(-x-cY')\sgn(-Y')]= F(x),$
i.e., $F$ is symmetric. 
Further, 
$F'(x) = \expect_y[g'(x+cY)\sgn(Y)]= \int_0^\infty (g'(x+cy)- g'(x-cy))\phi(y)dy$. 
The strict concavity of $g$ in $[0,\infty)$ implies that $g'$ is decreasing in $[0,+\infty)$, and the anti-symmetry of $g$ implies that $g'$ is symmetric. Take $x>0$: if $x>cy\geq 0$, $g'(x+cy)> g'(x-cy)$, while if $x\leq cy$, then $g'(x-cy) = g'(cy -x) > g'(cy+x) $, so $F'(x)$ is negative for $x>0$. By the symmetry of $F$, $F'(x)$ is positive for $x<0$. As such,  $F(x) =G(x^2)$ for some strictly decreasing $G$, and $\tilde{c}_\ell = \mathrm{Cov}(G(Z);Z)$ for $Z=X^2$; hence, $\tilde{c}_\ell<0$, for all $\ell\in [k]$.

To see that $\tilde{c}_\ell\neq 0$ for generic $\mu$, recall that $f$ is analytic and hence so is $g$. Hence, $\tilde{c}_\ell$ is an analytic function of $\mu$, for every $\ell\in [k]$; also, as $\tilde{c}_\ell(\mu)<0$  for $\mu=0$, it is not identically 0. Hence, the sets $\{\mu\in \reals^d: \tilde{c}_\ell(\mu)=0 \}$, $\ell\in [k]$, have Lebesgue measure 0 (see, e.g., pg. 83 in \citep{analytic}), and so does their union $Z$. As such, $\tilde{c}_\ell\neq 0$ for generic $\mu$; if not, there exists a ball $B\subset \reals^d$  such that $B\cap Z$ has positive Lebesgue measure, a contradiction. 
\end{proof}

\begin{figure}[!t]
\techrep{\includegraphics[width=0.55\columnwidth]{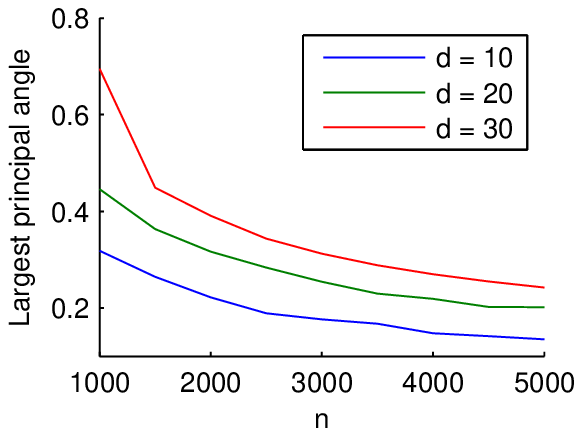}\!\!\!\!\!\!\!\!\!\!
\includegraphics[width=0.55\columnwidth]{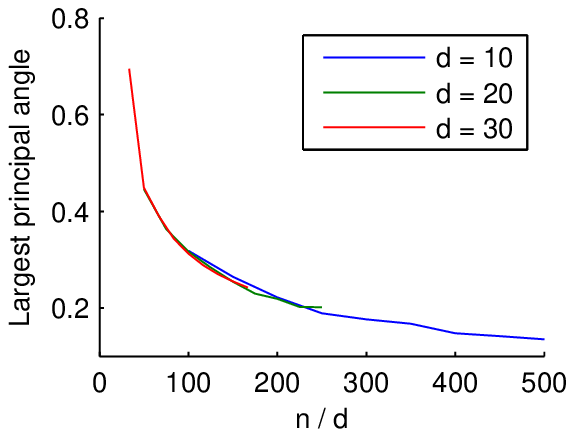}\!\!\!\!\!\!\\}{
\includegraphics[width=0.45\columnwidth]{PredictExpectLabelSubspaceLoss}\includegraphics[width=0.45\columnwidth]{PredictExpectLabelSubspaceLossScaled}\\
}
{\small \hspace*{\stretch{1}}(a) $\sin(d_P)$ vs.~$n$ \hspace*{\stretch{2}}(b) $\sin(d_P)$ vs.~$n/d$\hspace*{\stretch{1}}}
\vskip -0.1in
\caption{Convergence of $\hat{U}$ to $U$. }
\label{fig:subspace-loss}
\label{fig:scaled-subspace-loss}
\end{figure}

\begin{figure}
\techrep{\includegraphics[width=0.55\columnwidth]{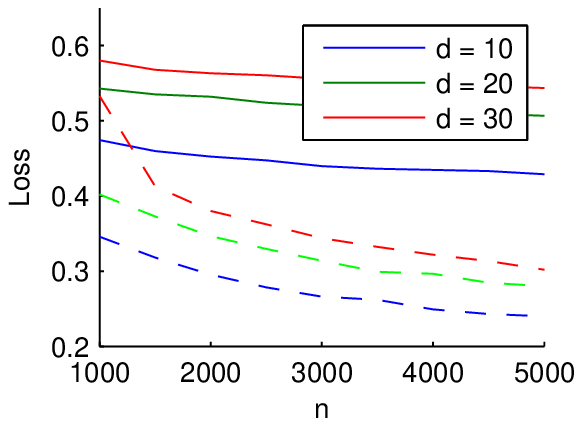}\!\!\!\!\!\!\!\!\!\!
\includegraphics[width=0.55\columnwidth]{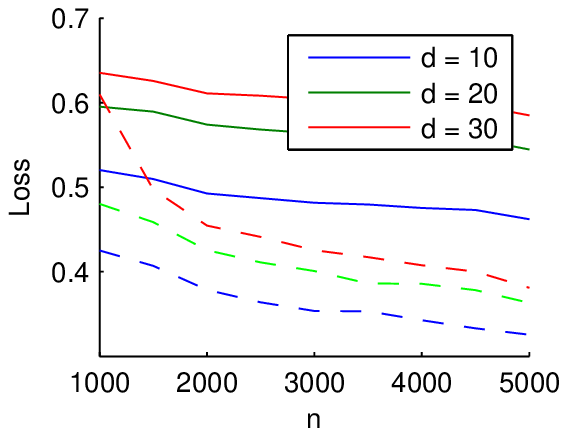}\!\!\!\!\!\!
\\}
{\includegraphics[width=0.45\columnwidth]{PredictExpectLabelLossSqrt}\includegraphics[width=0.45\columnwidth]{PredictExpectLabelLossLog}
\\}
{\small\hspace*{\stretch{1}}(a) $K=\sqrt{n}$ \hspace*{\stretch{2}}(b) $K = \log(n)$\hspace*{\stretch{1}}}
\vskip -0.1in
\caption{Predicting the expected label given features using $K$-NN (RMSE). Dotted lines are for $K$-NN after projecting the features $X_i$ onto   $\hat{U}$. }
\label{fig:knn-sqrt-loss}
\label{fig:knn-log-loss}
\end{figure}

Denote by $\lambda_0$ the eigenvalue of multiplicity $d-k$ in Lemma~\ref{lemma:spectrum}. Let $\Delta = \min_{\ell\in [k]}\abs{\lambda_0 - \lambda_\ell}$ be the gap between $\lambda_0$ and the remaining eigenvalues. Then, the following lemma holds; this, along with Lemma~\ref{lemma:spectrum}, yields Theorem~\ref{thm:main}. 
\begin{lemma} \label{thm:main2}
Let  $\hat{U}$ be our estimate for $U$. If $\lambda_1,\dots,\lambda_k$ are separated from $\lambda_{0}$ by at least $\Delta$, then for $\epsilon \le \min(\epsilon_0/\Delta,\frac14)$, we have
\begin{align*}
\Prob(d_P(U,\hat{U}) > \epsilon) 
\le C\exp\big(-F(\Delta\epsilon)\big),
\end{align*}
where $\epsilon_0$, $F$ are defined as per Lemma \ref{lem:q-large-deviation}.
\end{lemma}
\begin{proof}
If we ensure $\|\hat{Q} - Q\| \le \Delta/4$, then, by Weyl's theorem \citep{horn2012matrix}, $d-k$ eigenvalues of $\hat{Q}$ are contained in $[\lambda_{k+1} -\Delta/4, \lambda_{k+1} + \Delta/4]$, and the remaining eigenvalues are outside this set, and will be detected by \SM.  
Moreover, by the  Davis-Kahan $\sin(\theta)$ theorem, 
\begin{align*}
d_p(\mathrm{range}(Q),\mathrm{range}(\hat{Q})) &\le \frac{\|\hat{Q} \!-\! Q\|_2 }{\Delta - \|\hat{Q} \!-\! Q\|_2}
= \frac{1}{\frac{\Delta}{\|\hat{Q}\! - \!Q\|_2} \!-\! 1}.
\end{align*}
Thus the event $d_p(U,\hat{U}) \le \epsilon$ is implied by 
$
\|\hat{Q} - Q\|_2 \le \frac{\Delta\epsilon}{1 + \epsilon}\leq \Delta \epsilon.
$
 Moreover, this implies that sufficient condition for  $\|\hat{Q} - Q\|_2\le \Delta/4$ (which is required for \SM{} to detect the correct eigenvalues)  is that $\epsilon \le \frac14$. The lemma thus follows from Lemma~\ref{lem:q-large-deviation}.
\end{proof}
Note that the Gaussianity of $X$ is crucially used in the fact that the `whitened' features $W$ are uncorrelated, which in turn yields Eq.~\eqref{eq:breakdown}. We believe that the theorem can be extended to more general distributions, provided that the transform $\Sigma^{-\frac{1}{2}}$ de-correlates the coordinates of $X$. 

\section{Experiments}\label{sec:experiments}
\begin{figure*}[!t]
\includegraphics[width=0.32\textwidth]{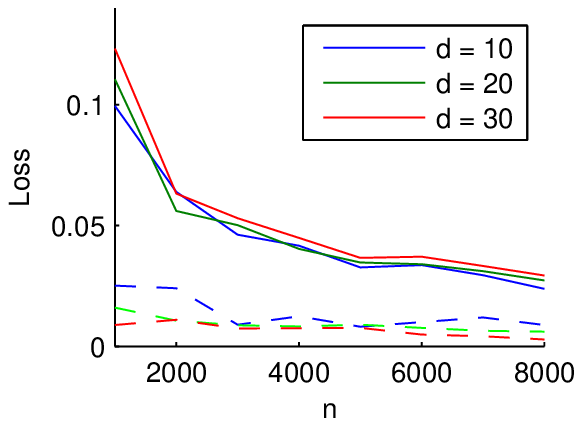}
\includegraphics[width=0.32\textwidth]{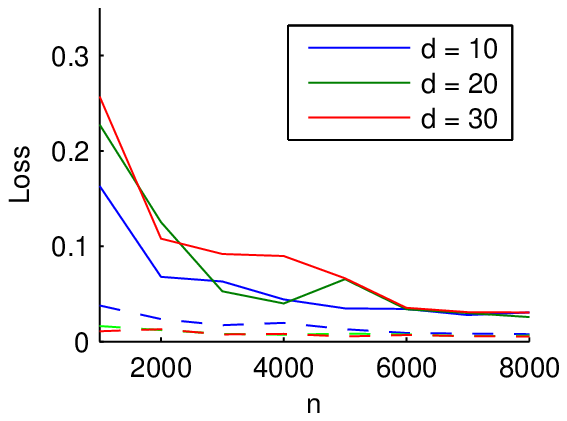}
\includegraphics[width=0.32\textwidth]{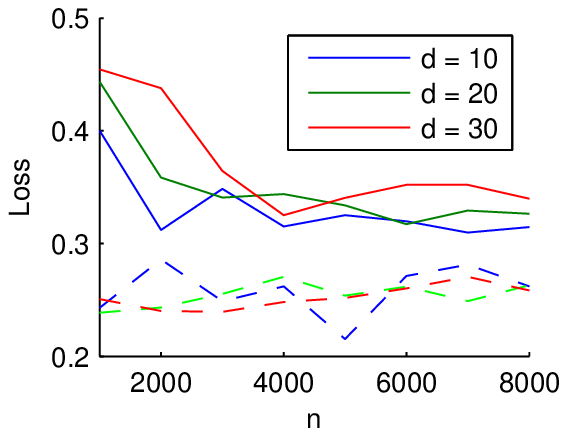}\\
{\techrep{\small}{\tiny}\hspace*{\stretch{1}}(a) EM Prediction (close to gr.~truth) \hspace*{\stretch{2}}(b)  EM Prediction (random)\hspace*{\stretch{2}} (c) Clustering (random) \hspace*{\stretch{1}}}
\caption{(a) Predicting the label given features and the classifier using using EM (normalized 0-1 loss) from a starting point close to ground truth.  Dotted lines are for kNN after projecting the features onto the estimated subspace. (b) Predicting the label given features and the classifier using using EM (normalized 0-1 loss) from a random starting point. (c) Predicting the classifier given features and the label.
}\label{fig:em-loss}\label{fig:em-rand-loss}
\label{fig:identify-users-loss}
\label{fig:identify-users-rand-loss}
\vskip -0.2in
\end{figure*}

We conduct computational experiments to validate the performance of \textsc{SpecralMirror} on subspace estimation, prediction, and clustering. We generate 
synthetic data using  $k = 2$, with profiles  $u_\ell\sim\cN(0,I),\,\ell = 1,2$ and mixture weights $p_\ell$  sampled uniformly at random from  the $k$-dimensional simplex. Features are also Gaussian: $X_i \sim\cN(0,I),\,i = 1,\dots,n$; labels generated by the $\ell$-th classifier are given by  $y_i = \sign(u_\ell^TX_i),\,i = 1,\dots,n$. 

\noindent\textbf{Convergence.} We study first how well \SM estimates the span $U$. Figure 
\ref{fig:subspace-loss}(a) shows the convergence of 
$\hat{U}$ to $U$ in terms of (the sin of) the 
largest principal angle between the subspaces versus the sample size $n$. We also plot the convergence versus the 
effective sample size $n / d$ (Figure \ref{fig:scaled-subspace-loss}(a)).  The curves for different values of $d$
align in Figure \ref{fig:scaled-subspace-loss}, indicating that the upper bound in Thm.~\ref{thm:main} correctly predicts the sample complexity as $n\approx \Theta(d)$.
Though not guaranteed by Theorem~\ref{thm:main}, in all experiments $r$ was indeed spanned by $\hat{U}$, so the addition of $\hat{r}$ to $\hat{U}$ was not necessary.

\noindent\textbf{Prediction through $K$-NN.} Next, we use the estimated subspace to aid in the prediction of expected labels. 
Given a new feature vector $X$, we use the average label of its $K$ nearest
neighbors ($K$-NN) in the training set to predict its expected label. We do this for two settings: once over the raw data (the `ambient' space), and once over data for which the features $X$ are first projected to $\hat{U}$, the estimated span (of dimension 2).  For each $n$, we repeat this procedure 25 times with $K=\sqrt{n}$ and $K=\log n$.  We record the average  root mean squared error between predicted and expected labels over the 25 runs. Figures \ref{fig:knn-sqrt-loss}(a) and \ref{fig:knn-log-loss}(b) show that, despite  
the error in $\hat{U}$, using $K$-NN on this subspace outperforms $K$-NN
on the ambient space.



\noindent\textbf{Prediction and Clustering through EM.} We next study the performance of prediction and clustering using the Expectation-Maximization (EM) algorithm.
We use EM to fit the individual profiles 
both over the training set, as well as on the dataset projected to the estimated subspace $\hat{U}$. 
 We conducted two 
experiments in this setting:
(a) initialize EM close to the true  profiles $u_\ell$, $\ell\in [k]$, and (b) randomly initialize EM and choose the best set of profiles from 30 
runs. For each $n$ we run EM 10 times. 

The first set of prediction experiments, we  again compare expected labels to the predicted labels, using for the latter  profiles $u_\ell$ and mixture probabilities $p_\ell$ as estimated by EM.
  Figure~\ref{fig:em-loss}(a) measures the statistical efficiency of EM over the 
estimated subspace versus EM over the ambient space, when EM is initialized close to the true profiles.
The second set of experiments, illustrated in Figure~\ref{fig:em-rand-loss}(b), aims to capture the additional improvement due to the reduction in the number of 
local minima in the reduced space. In both cases we see that constraining the estimated profiles to lie in the estimated subspace improves the statistical efficiency of EM; in the more realistic random start experiments, enforcing the subspace constraint also improves the performance of EM by reducing the number of local minima. We also observe an overall improvement compared to prediction through $K$-NN.

Finally, we use the fitted profiles $u_\ell$ to identify the classifier generating a label given the features and the label. To do this, once the profiles $u_\ell$ have been detected by EM, we use a logistic model margin condition to identify the classifier who generated a label,  given the label and its features. Figure \ref{fig:identify-users-rand-loss}(c) shows the result for  EM initialized at a random point, after choosing the best set of profiles from out of 30 runs. We evaluate the performance of this clustering procedure using the normalized 0-1 loss. Again, constraining the estimated profiles to the estimated subspace significantly improves the performance of this clustering task. 



\section{Conclusions}\label{sec:conclusions}
We have proposed \SM, a method for discovering the span of a mixture of linear classifiers. Our method relies on a non-linear transform of the labels, which we refer to as `mirroring'. Moreover, we have provided consistency guarantees and non-asymptotic bounds, that also imply the near optimal statistical efficiency of the method. Finally, we have shown that, despite the fact that \SM discovers the span only approximately, this is sufficient to allow for a significant improvement in both prediction and clustering, when the features are projected to the estimated span.

We have already discussed several technical issues that remain open, and that we believe are amenable to further analysis. These include 
 amending the Gaussianity assumption, and applying our bounds to other pHd-inspired methods.
 An additional research topic is to further improve the computational complexity of the estimation of the eigenvectors of the `mirrored' matrix $\hat{Q}$. This is of greatest interest  in cases where the covariance $\Sigma$ and mean $\mu$ are a priori known. This would be the case when, e.g., the method is applied repeatedly  and, although the features $X$ are sampled from the same distribution each time, labels $Y$ are generated from a different mixture of classifiers. In this case, $\SM$ lacks the pre-processing step, that requires estimating $\Sigma$ and is thus computationally intensive; the remaining operations amount to discovering the spectrum of $\hat{Q}$, an operation that can be performed more efficiently. For example, we can use a regularized M-estimator to exploit the fact that $\Sigma^{-1/2}\hat{Q}\Sigma^{-1/2}$ should be the sum of a multiple of the identity and a low rank matrix---see, e.g., \citet{negahban2012unified}.

\bibliography{references}
\bibliographystyle{icml2014}


\newpage
\appendix

\section{A Large-Deviation Lemma}

We first prove a Bernstein-type inequality for sub-Gaussian random vectors, that we shall use in our proofs:
\begin{lemma}
\label{lem:large-deviation-sub-Gaussian-vector}
Let $X\in\reals^d$ be a sub-Gaussian random vector, \ie\ $\langle a,X\rangle$ is sub-Gaussian for any $a\in\reals^d$. Then there exist universal constants $c_1,c_2$ such that
\begin{align*}
\Prob(\norm{X}_2 \ge t)& \le\\ c_1\exp\bigg(-\min&\left\{\frac{c_2(t^2 - d\norm{\Sigma}_2)}{4d\norm{X}_{\psi_2}^2},\frac{(t^2 - d\norm{\Sigma}_2)^2}{64c_2\norm{X}_{\psi_2}^4}\right\}\bigg).
\end{align*}
\end{lemma}

\begin{proof}
By the (exponential) Markov inequality, we have
\BEQ\begin{split}
\Prob(\norm{X}_2 \ge t) = \Prob(\exp(\lambda\norm{X}_2^2) \ge \exp(\lambda t^2))\\ \le \frac{\Expect[\exp(\lambda\norm{X}_2^2)]}{\exp(\lambda t^2)}.\end{split}
\label{eq:large-deviation-sub-Gaussian-vector-1}
\EEQ
Let $Z$ be uniformly distributed on the unit sphere $\bS^{d-1}$. Then $\sqrt{d}Z$ is isotropic so $d\Expect[\langle Z,a\rangle^2] = \norm{a}_2^2$ for all $a$. We use this fact to bound the m.g.f. of $\norm{X}_2^2$:
\begin{align*}
\Expect[\exp(\lambda\norm{X}_2^2)] &= \Expect_X[\exp(\lambda d\Expect_{Z}[\langle Z,X\rangle^2])] \\
&\le  \Expect_X[\Expect_{Z}[\exp(\lambda d\langle Z,X\rangle^2)]].
\end{align*}
We interchange the order of expectation to obtain
\begin{align*}
\Expect[\exp(\lambda\norm{X}_2^2)] &\le \Expect_Z[\Expect_{X}[\exp(\lambda d\langle X,Z\rangle^2)]] \\
&\le \sup\,\{\Expect_{X}[\exp(\lambda d\langle X,Z\rangle^2)]\mid z\in\bS^{d-1}\} .
\end{align*}
$\langle X, z\rangle$ is sub-Gaussian (for fixed $z$) so $\langle X, z\rangle^2$ is (noncentered) sub-exponential. If $\lambda d < c/\norm{\langle X, z\rangle^2 - \Expect[\langle X, z\rangle^2]}_{\psi_1}$, then
\begin{align}
&\Expect[\exp(\lambda d\langle X, z\rangle^2)]\nonumber \\
&\le \exp(\lambda d\Expect[\langle X, z\rangle^2])\Expect[\exp(\lambda d(\langle X, z\rangle^2 - \Expect[\langle X, z\rangle^2]))] \nonumber\\
&\le \exp(\lambda d\Expect[\langle X, z\rangle^2] + c d^2\lambda^2\norm{\langle X, z\rangle^2 - \Expect[\langle X, z\rangle^2]}_{\psi_1}^2)\nonumber \\
&\le \exp(\lambda d\Expect[\langle X, z\rangle^2] + 4c d^2\lambda^2\norm{\langle X, z\rangle^2}_{\psi_1}^2) \nonumber \\
&\le \exp(\lambda d\Expect[\langle X, z\rangle^2] + 16c d^2\lambda^2\norm{\langle X, z\rangle}_{\psi_2}^4).
\label{eq:mgf-bound}
\end{align}
We substitute this bound into \eqref{eq:large-deviation-sub-Gaussian-vector-1} to obtain
$$
\Prob(\norm{X}_2 \ge t) \le\\ \exp(16c d^2\lambda^2\norm{X}^4_{\psi_2} + \lambda (d\norm{\Sigma}_2 - t^2)),
$$
where $\Sigma$ is the covariance matrix of $X$. We optimize over $\lambda$ to obtain
$$
\Prob(\norm{X}_2 \ge t) \le e^{-\frac{(t^2 - d\norm{\Sigma}_2)^2}{64c d^2\norm{X}_{\psi_2}^4}}.
$$
If the optimum lies outside the region where the m.g.f.\ bound holds \eqref{eq:mgf-bound}, we can always choose 
$$
\lambda = \frac{c}{4d\norm{X}_{\psi_2}^2} \le \frac{c}{d\norm{\langle X, z\rangle^2 - \Expect[\langle X, z\rangle^2]}_{\psi_1}}
$$
in to obtain the tail bound:
$$
\Prob(\norm{X}_2 \ge t) \le e^{c^3 + \frac{c(d\norm{\Sigma}_2 - t^2)}{4d\norm{X}_{\psi_2}^2}}.
$$
We combine these two bounds to obtain
$$
\Prob(\norm{X}_2 \ge t) \le c_1e^{-\min\left\{\frac{c_2(t^2 - d\norm{\Sigma}_2)}{4d\norm{X}_{\psi_2}^2},\frac{(t^2 - d\norm{\Sigma}_2)^2}{64c_2\norm{X}_{\psi_2}^4}\right\}}.
$$
Note that the $t^2$ bound always holds. However, for small $t$, the $t^4$ term yields a tighter bound.
\end{proof}

\section{Proof of Lemma~{\protect{\lowercase{\ref{lemma:hatr_convergence}}}} (Weak Conv.~of $\hat{r}$)}\label{app:hatr}
\begin{proof}
We expand $\hat{r}_n - r$ (and neglect higher order terms) to obtain
\begin{align}
\norm{\hat{r}_n - r}_2 &= \Bigl\|\frac{1}{n}\sum_{i=1}^n\hat{\Sigma}^{-1} Y_i(X_i - \hat{\mu}) - \Sigma^{-1}\Expect[s(X)(X-\mu)]\Bigr\|_2\nonumber \\
&\le \norm{\hat{\Sigma}^{-1}}_2\Bigl\|\frac{1}{n}\sum_{i=1}^n Y_i(X_i - \hat{\mu}) - \Expect[s(X)(X-\mu)]\Bigr\|_2 \nonumber\\
&\pc+\|\Expect[s(X)(X-\mu)]\|_2\|\hat{\Sigma}^{-1} - \Sigma^{-1}\|_2 + (o_P(1))^2.
\label{eq:r-1}
\end{align}
The higher order terms generically look like
\BEQ
\Prob(\abs{X - \Expect[X]}\abs{Y - \Expect[Y]} > \epsilon).
\label{eq:composite-terms}
\EEQ
We apply the union bound to deduce
\begin{align*}
&\Prob(\abs{X - \Expect[X]}\abs{Y - \Expect[Y]} > \epsilon) \\
&\pc\le \Prob(\abs{X - \Expect[X]} > \sqrt{\epsilon}) +  \Prob(\abs{X - \Expect[X]} > \sqrt{\epsilon}).
\end{align*}
For any $\epsilon <  1$, $\sqrt{\epsilon} > \epsilon$ and we have
$$
\Prob(\abs{X - \Expect[X]} > \sqrt{\epsilon}) \le \Prob(\abs{X - \Expect[X]} > \epsilon).
$$
Since terms of the form $\Prob(\abs{X - \Expect[X]} > \epsilon)$ appear in the upper bounds we derive, we can handle terms like \eqref{eq:composite-terms} with a constant factor (say 2). Since our bounds involve multiplicative constant factors anyways, we neglect these terms  to simplify our derivation.

We expand the first term to obtain
\begin{align*}
&\Bigl\|\frac{1}{n}\sum_{i=1}^n Y_i(X_i - \hat{\mu}) -\Expect[Y(X-\mu)]\Bigr\|_2 \\
&\pc\le |\Expect[s(X)]|\norm{\hat{\mu} - \mu}_2 + \norm{\mu}_2\Bigl|\frac{1}{n}\sum_{i=1}^n Y_i - \Expect[s(X)]\Bigr|_2 \\
&\pc\pc+ \Bigl\|\frac{1}{n}\sum_{i=1}^n Y_i(X_i - \mu) - \Expect[s(X)(X-\mu)]\Bigr\|_2 \\
&\pc\pc + (o_P(1))^2.
\end{align*}
$\hat{\mu} - \mu$ is a sub-Gaussian random variable with sub-Gaussian norm $\frac{\norm{X}_{\psi_2}}{\sqrt{n}}$, so there exist universal $c_1$ and $c_2$ s.t.\
$$
\Prob(\norm{\hat{\mu} - \mu}_2 > t) \le c_1\exp\left(-\frac{c_2n(t^2 - d\norm{\Sigma}_2)}{4d\norm{X}_{\psi_2}^2}\right).
$$ 
$Y$ is bounded between 1 and -1, so 
\BNUM
\item We can use Chernoff's inequality to deduce
$$
\Prob\biggl(\Bigl|\frac{1}{n}\sum_{i=1}^n Y_i - \Expect[s(X)]\Bigr| > t \bigr) \le 2\exp(-nt^2/2).
$$
\item $Y_i(X_i - \mu)$ are sub-Gaussian. Thus there exist universal $c_1$ and $c_2$ such that
\begin{align*}\begin{split}
\Prob\biggl(\Bigl\|\frac{1}{n}\sum_{i=1}^n Y_i(X_i - \mu) - \Expect[s(X)(X - \mu)]\Bigr\|_2 > t \bigr) \le\\ c_1\exp\left(-\frac{c_2n(t^2 - d\norm{\Sigma}_2)}{4d\norm{X}_{\psi_2}^2}\right).\end{split}
\end{align*}
\ENUM

We expand the second term in \eqref{eq:r-1} to obtain
\begin{align*}
\|\hat{\Sigma}^{-1} - \Sigma^{-1}\| = \|\Sigma^{-1/2}\|\|\Sigma^{1/2}\hat{\Sigma}^{-1}\Sigma^{1/2} - I\|\|\Sigma^{-1/2}\|.
\end{align*}
We expand the middle term to obtain
\begin{align*}
&\|\Sigma^{-1/2}\hat{\Sigma}\Sigma^{-1/2} - I\| \\
&\pc\le \Bigl\|\Bigl(\frac1n\sum_{i=1}^n \Sigma^{-1/2}(X_i - \mu)(X_i - \mu)^T\Sigma^{-1/2}\Bigr)^{-1} - I\Bigr\|_2 \\
&\pc\pc+2\norm{\mu}_2\norm{\hat{\mu} - \mu}_2 + (o_P(1))^2
\end{align*}
We use Theorem 5.39 in \citet{vershynin2010introduction} to bound the first term:
$$
\Prob\biggl(\Bigl\|\frac1n\sum_{i=1}^n W_iW_i^T - I\Bigr\|_2 >  t\Bigr) \le 2\exp(-c_1'(\sqrt{n}t - c_2'\sqrt{d})^2),
$$
where $c_1',c_2'$ depend on the sub-Gaussian norm of $W$. We substitute these bounds into our expression for $\|\hat{r} - r\|_2$ to deduce
$$\textstyle
\Prob(\|\hat{r}\! -\! r\|_2 \ge \epsilon) \le Ce^{-\min\left\{\frac{c_1n\epsilon^2}{d},\bigl(c_1'\sqrt{n}\epsilon - c_2'\sqrt{d}\bigr)^2\right\}
},$$
where $C$ is an absolute constant and $c_1, c_1',c_2'$ depend on the sub-Gaussian norm of $X$. 
\end{proof}

\section{Proof of Lemma {\protect{\lowercase{\ref{lem:q-large-deviation}}}}  (Weak Conv.~of $\hat{Q}$)}
\label{app:lemmQ}

Let $\tilde{r}$ denote the projection of $r$ onto $U$.

\begin{lemma}
\label{lem:r-lemma}
If $$\|\hat{r} - r\|_2 \le \epsilon_0 = \min\,\{\alpha_1,\ldots,\alpha_k\}\sigma_{\min}(U),$$ then $\tilde{r}$ lies in the interior of the positive cone spanned by $(u_\ell)_{\ell\in [k]}$, where $\sigma_{\min}(U)$ is the smallest nonzero singular value of $U$. 
\end{lemma}

\begin{proof}
$r$ lies in interior of the conic hull of $\{u_1,\dots,u_k\}$, so we can express $r$ as $\sum_{i=1}^k\alpha_iu_i$, where $r_i > 0$. If $\tilde{r}$ also lies in the conic hull, then $\tilde{r} = \sum_{i=1}^k\beta_iu_i$ for some $\beta_i > 0$. Then
\begin{align*}
&\norm{\tilde{r} - r}_2 = \|\sum_{i=1}^k(\alpha_i - \beta_i)u_i\|_2 = \sqrt{(\alpha - \beta)^TUU^T(\alpha - \beta)} \\
&\ge \norm{\alpha - \beta}_2 \sigma_{\min}(U) \ge \norm{\alpha - \beta}_\infty \sigma_{\min}(U).
\end{align*}
To ensure $\beta$ is component-wise positive, we must have $\norm{\alpha - \beta}_\infty < \min\{\alpha_1,\dots,\alpha_k\}$. A sufficient condition is
$\norm{\tilde{r} - r}_2 \leq \epsilon_0 \le \min\{\alpha_1,\dots,\alpha_k\}\sigma_{\min}(U)$.
\end{proof}

We are now ready to prove Lemma 3. We expand $\|\hat{Q}_n - Q\|_2$ (and neglect higher order terms) to obtain
\begin{align*}
\|\hat{Q}_n &- Q\|_2 \le\\& \Bigl\|\frac1n\sum_{i=1}^nZ_i\Sigma^{-1/2}(X_i - \mu)(X_i - \mu)^T\Sigma^{-1/2} - Q\Bigr\|_2 \\
&\pc + 2\|\hat{\Sigma}^{-1/2} - \Sigma^{-1/2}\|_2\Expect\bigl[\|(X-\mu)(X-\mu)\Sigma^{-1/2}\|_2\bigr] \\
&\pc + 2\|\Sigma^{-1/2}\|_2\|\hat{\mu} - \mu\|_2\Expect\bigl[\|(X-\mu)\Sigma^{-1/2}\|_2\bigr] \\
&\pc + (o_P(1))^2.
\end{align*}
The second and third terms can be bounded using the same bounds used in the analysis of how fast $\hat{r}$ converges to $r$. Thus we focus on how fast 
$$
\sum_{i=1}^n Z_i\Sigma^{-1/2}(X_i - \mu)(X_i - \mu)^T\Sigma^{-1/2} 
$$ 
coverges to $Q$. 
Let $\epsilon'=\min(\epsilon_0,\frac{r}{2})$.
First, we note that
\begin{align}
&\Prob\biggl(\Bigl\|\frac1n\sum_{i=1}^n Z_iW_iW_i^T - Q\Bigr\|_2 > t \Bigr)\label{bayes} \\
&\pc\le \Prob\biggl(\Bigl\|\frac1n\sum_{i=1}^n Z_iW_iW_i^T - Q\Bigr\|_2 > t\mid\hat{r}\in B_{\epsilon'}(r) \Bigr) \Prob(\hat{r}\in B_{\epsilon'}(r))\nonumber \\
&\pc\pc+ \Prob(\hat{r}\notin B_{\epsilon'}(r)).\nonumber
\end{align}
Let $\tilde{Z}_i$ denote the ``corrected'' version of the $Z_i$'s, \ie\ the $Z_i$'s we obtain if we use the projection of $\hat{r}$ onto $U$ to flip the labels, and $W_i$ denote $\Sigma^{-1/2}(X_i-  \mu)$. We have
\begin{align}
\label{eq:corrected-z}
&\Bigl\|\frac1n\sum_{i=1}^n Z_iW_iW_i^T - Q\Bigr\|_2  \\
&\pc\le \Bigl\|\frac1n\sum_{i=1}^n \tilde{Z}_iW_iW_i^T - Q\Bigr\|_2 + \Bigl\|\frac1n\sum_{i=1}^n (Z_i-\tilde{Z}_i)W_iW_i^T\Bigr\|_2.\nonumber
\end{align}
The probability the first term is large is bounded by
\begin{align}
&\Prob\biggl(\Bigl\|\frac1n\sum_{i=1}^n \tilde{Z}_iW_iW_i^T - Q\Bigr\|_2 > t \mid \hat{r}\in B_{\epsilon'}(r)\Bigr)\label{eq:first-term-large}
 \\
&\pc\le \sup_{\hat{r}\in B_{\epsilon_0}(r)}\Prob\biggl(\Bigl\|\frac1n\sum_{i=1}^n \tilde{Z}_iW_iW_i^T - Q\Bigr\|_2 > t\mid\hat{r}\in B_{\epsilon'}(r)\Bigr)\nonumber
\end{align}
The $Z_i$'s are independent of $W_i$'s because the $Z_i$'s were computed using $\hat{r}$ that was in turn computed independently of the $X_i$'s, as the former are computed on a different partition of $[n]$. Thus,  the sum in the r.h.s.~of \eqref{eq:corrected-z} is a sum of \iid\ r.v. and can be bounded by
\begin{align*}
&\sup_{\hat{r}\in B_{\epsilon_r}(r)}\Prob\biggl(\Bigl\|\frac1n\sum_{i=1}^n \tilde{Z}_iW_iW_i^T - Q\Bigr\|_2 > t\mid\hat{r}\in B_{\epsilon'}(r)\bigr) \\
&\pc\le 2\exp(-c_1(\sqrt{n}t - c_2\sqrt{d})^2),
\end{align*}
where $c_1,c_2$ depend on the sub-Gaussian norm of $\tilde{Z}W$. This is a consequence of Remark 5.40 in \citet{vershynin2010introduction}.

We now focus on bounding the second term in \eqref{eq:corrected-z}. 
In what follows, without loss of generality,  we will restrict $W$ to the k+1 subspace spanned by $U$ and $\hat{r}$, as remaining components of the $W_i$'s do not contribute to the computation. Let $C_{\hat{r}}$ (for cone) be the ``bad'' region, \ie\ the region where $Z \ne \tilde{Z}$. We have
$$
\Bigl\|\frac1n\sum_{i=1}^n (Z_i-\tilde{Z}_i)W_iW_i^T\Bigr\|_2 =\Bigl\|\frac2n\sum_{i=1}^n \ones_{C_{\hat{r}}}(W_i)W_iW_i^T\Bigr\|_2.
$$
By the triangle inequality, we have
\begin{align}
&\Bigl\|\frac2n\sum_{i=1}^n \ones_{C_{\hat{r}}}(W_i)W_iW_i^T\Bigr\|_2 \nonumber \\
&\pc \le 2\norm{\Expect[\ones_{C_{\hat{r}}}(W)WW^T]}_2 + \Bigl\|\frac2n\sum_{i=1}^n &\ones_{C_{\hat{r}}}(W_i)W_iW_i^T\nonumber\\&\pc\pc- 2\Expect[\ones_{C_{\hat{r}}}(W)WW^T] \Bigr\|_2
\label{eq:triangle-inequality}
\end{align}
$\ones_{C_{\hat{r}}}$ is bounded, hence $\ones_{C_{\hat{r}}}(W_i)W_i$ is sub-Gaussian and
\begin{align}
\Prob\bigg(\Bigl\|\frac2n\sum_{i=1}^n &\ones_{C_{\hat{r}}}(W_i)W_iW_i^T\nonumber\\&- 2\Expect[\ones_{C_{\hat{r}}}(W)WW^T] \Bigr\|_2> t\mid\hat{r}\in B_{\epsilon'}(r)\bigg)\nonumber \\
&\pc\le 2\exp(-c_1(\sqrt{n}t - c_2\sqrt{d})^2),
\label{eq:high-prob-C}
\end{align}
where $c_1,c_2$ depend on the sub-Gaussian norm of $X$. It remains to bound $\norm{\Expect[\ones_{C_{\hat{r}}}(W)WW^T]}_2$. 

We use Jensen's inequality to obtain
\begin{align}
&\pc\norm{\Expect[\ones_{C_{\hat{r}}}(W)WW^T]}_2 \le \Expect[\ones_{C_{\hat{r}}}(W)\norm{WW^T}_2] \nonumber\\
&\pc\le C\Prob(W\in C_{\hat{r}}),
\label{eq:expression-C-P}
\end{align}
where the constant $C$ depends on $k$, as $W$ is restricted to the space spanned by $U$ and $\hat{r}$.
Finally, we bound $\Prob(W\in C_{\hat{r}})$. 

The distribution of $W_i$ is spherically symmetric so the probability that $W_i\in C_{\hat{r}}$ is proportional to the surface area of the set $C_{\hat{r}}\cap \bS^k$, where $\bS^k$ the unit sphere centered at the origin:
$$
C_{\hat{r}}\cap \bS^k = \{w\in\bS^k\mid \hat{r}^Tw \le 0, u_1^Tw, \dots, u_k^Tw \ge 0\}.
$$
Lemma~\ref{lemma:quad} implies that this set is contained in the set
\[
S = \{w\in\bS^k\mid \hat{r}^Tw \le 0, r^Tw \ge 0\}.
\]
By a symmetry argument, the volume of $S$ (according to the normalized measure on the unit sphere) is simply the angle between $\hat{r}$ and $r$, \ie
\[
\Prob(W\in C_{\hat{r}}) \le \arccos(\frac{\hat{r}^Tr}{\|\hat{r}\|\|r\|})
\]
Let $\epsilon_r\equiv\norm{\hat{r} - r}$. Recall that $\epsilon_r\leq r/2$, by conditioning; it can be verified that this implies $\frac{\hat{r}^Tr}{\|\hat{r}\|\|r\|}\ge 1 - \frac{2\epsilon_r}{\min( \|r\|,\|r\|^2)}$. Thus
\[
\Prob(W\in C_{\hat{r}}) \le  \arccos(1-c_{\|r\|} \epsilon_r)
\]
where $c_{\|r\|}$ depends on $\|r\|$. 
We substitute these bounds into \eqref{eq:expression-C-P} to obtain
\begin{align}
\norm{\Expect[\ones_{C_{\hat{r}}}(W)WW^T]}_2 \le C \arccos(1-c_{\|r\|}{\epsilon_r}).\label{crbound}
\end{align}
We combine \eqref{crbound} with \eqref{eq:high-prob-C} to deduce
\begin{align*}
&\Prob(\Bigl\|\frac1n\sum_{i=1}^n (Z_i-\tilde{Z}_i)W_iW_i^T\Bigr\|_2 > C\arccos(1-c_{\|r\|}\epsilon_r) + t\mid\hat{r}\in B_{\epsilon'}(r)) \\
&\pc\le 2\exp(-c_1(\sqrt{n}t - c_2\sqrt{d})^2)
\end{align*}
Using this expression and \eqref{eq:first-term-large}, a union bound on Ineq.~\eqref{eq:corrected-z} gives:
\begin{align}
&\Prob(\|\hat{Q} \!- \!Q\|_2 > C\arccos\big(1 \!-\!c_{\|r\|} \epsilon_r\big) + t\mid\hat{r}\in B_{\epsilon'}(r))\nonumber \\
&\pc\le C\exp\left(-\min\left\{\frac{c_1nt^2}{d},\bigl(c_1'\sqrt{n}t - c_2'\sqrt{d}\bigr)^2\right\}\right).\label{weird}
\end{align}
for appropriate constants $C$,$c_1$,$c_1'$,$c_2'$. Note that
\begin{align*}
&\Prob(\|\hat{Q} - Q\|_2 > \epsilon\mid\hat{r}\in B_{\epsilon'}(r))\Prob(\hat{r}\in B_{\epsilon'}(r)) \\
&\le \Prob(\|\hat{Q} - Q\|_2 > \epsilon \mid C\arccos(1-c_{\|r\|}\epsilon_r)\leq \frac{\epsilon}{2}, \hat{r}\in B_{\epsilon'}(r)) \\
&+ \Prob(\norm{\hat{r} - r}_2 > \frac{1}{c_{\|r\|}}(1- \cos(\epsilon/C))),
\end{align*}
Let 
$$
f(\epsilon) := \min\left\{\frac{c_1n\epsilon^2}{d},\bigl(c_1'\sqrt{n}\epsilon - c_2'\sqrt{d}\bigr)^2\right\}.
$$
The first term is bounded by $C\exp\left(-f(\epsilon/2)\right)$ and the second term is bounded by $C\exp\left(-f(     \frac{1}{c_{\|r\|}}(1- \cos(\epsilon/C))       )\right)$. We deduce
\begin{align*}
&\Prob(\|\hat{Q} - Q\|_2 > \epsilon\mid\hat{r}\in B_{\epsilon'}(r))\Prob(\hat{r}\in B_{\epsilon'}(r)) \\
&\pc\le C(\exp(-f( \frac{1}{c_{\|r\|}}(1- \cos(\epsilon/C)))) + \exp(-f(\epsilon/2)))\\
&\pc\le C\exp(-f(\min\{     \frac{1}{c_{\|r\|}}(1- \cos(\epsilon/C))        ,\epsilon/2\})).
\end{align*}
To complete the proof of Lemma \ref{lem:q-large-deviation}, one can similarly account for the event $\norm{\hat{r} - r}_2 \ge \epsilon'$ in \eqref{bayes}, finaly yielding:
\begin{align*}
&\Prob(\|\hat{Q} - Q\|_2 > \epsilon) \\
&\pc\le C\exp(-f(\min\{     \frac{1}{c_{\|r\|}}(1- \cos(\epsilon/C))        ,\epsilon/2, \epsilon'\})).
\end{align*}
The lemma thus follows from the fact that
$1-\cos(\epsilon) \approx \frac{1}{2}\epsilon^2$ for small enough $\epsilon>0$.

\section{Principal Hessian Directions}\label{app:phd}

In this section, we apply the principal Hessian directions (pHd) \citep{li1992principal} method to our setting, and demontrate its failure to discover the space spanned by parameter profile when $\mu =0$. Recall that pHd considers a setting in which  features $X_i\in \reals^d$ are i.i.d.~and normally distributed with mean $\mu$ and covariance $\Sigma$, while labels $Y_i\in R$ lie, in expectation, in a $k$-dimensional manifold. In particular, some smooth $h:\reals^k\to\reals$, $k\ll d$:
$$\expect[Y\mid X= x] = h(\<u_1,x\>,\ldots,\<u_k,x\>)$$
where $u_\ell\in \reals^d$, $\ell \in [k]$. The method effectively creates an estimate 
$$\hat{H} =
n^{-1}\sum_{i=1}^nY_i\, \Sigma ^{-\frac12}X_iX_i^T\Sigma^{-\frac12}\in\reals^{d\times d}
$$
of the Hessian
\begin{align}\begin{split}H=\expect[\nabla_x^2 h(\<u_1,X\>,\ldots,\<u_k,X\>)] =\\ U^T \expect[\nabla^2_v h(\<u_1,X\>,\ldots,\<u_k,X\>)]  U,\end{split}\label{pHd:hessian} \end{align} where $\nabla_v^2h$ is the Hessian of the mapping $v\mapsto h(v)$, for $v\in \reals^k$, and $U$ the matrix of profiles. As in our case, the method discovers $\linspan(u_1,\ldots,u_k)$ from the eigenvectors of the eigenvalues that ``stand out'', after appropriate rotation in terms of $\Sigma$.

 Unfortunately, in the case of linear classifiers, 
$$h(v) =  \sum_{\ell=1}^kp_\ell g(v_\ell) $$ 
for $g(s) = 2f(s)-1$ is anti-symmetric. As a result, $\nabla^2_v h$ is a diagonal matrix whose $\ell$-th entry in the diagonal is $g''(v_\ell)$. Since $g$ is anti-symmetric, so is $g''$. Hence, if $\mu = 0$ we have that $\expect[g''(\langle u_\ell,X \rangle)]=0$; hence, the Hessian $H$ given by \eqref{pHd:hessian} will in fact be zero, and the method will fail to discover any signal pertaining to $\linspan(U)$.

This calls for an application of the pHd method to a transform of the labels $Y$. This ought to be non-linear, as an affine transform would preserve the above property. Moreover, given that these labels are binary,  polynomial transforms do not add any additional signal to $Y$, and are therefore not much help in accomplishing this task. In contrast, the `mirrorring' approach that we propose  provides a means of transforming the labels so that their expectation indeed carries sufficient information to extract the span $U$, as evidenced by Theorem~1.


\end{document}


\setcounter{lemma}{6}
\twocolumn[
\icmltitle{Supplementary Material to ``Learning Mixtures of Linear Classifiers''}

\icmlauthor{Yuekai Sun}{yuekai@stanford.edu}
\icmladdress{Stanford University}
\icmlauthor{Stratis Ioannidis}{stratis.ioannidis@technicolor.com}
\icmladdress{Technicolor} 
\icmlauthor{Andrea Montanari}{montanar@stanford.edu}
\icmladdress{Stanford University} 


]

\section{A Large-Deviation Lemma}

We first prove a Bernstein-type inequality for sub-Gaussian random vectors, that we shall use in our proofs:
\begin{lemma}
\label{lem:large-deviation-sub-Gaussian-vector}
Let $X\in\reals^d$ be a sub-Gaussian random vector, \ie\ $\langle a,X\rangle$ is sub-Gaussian for any $a\in\reals^d$. Then there exist universal constants $c_1,c_2$ such that
\begin{align*}
\Prob(\norm{X}_2 \ge t)& \le\\ c_1\exp\bigg(-\min&\left\{\frac{c_2(t^2 - d\norm{\Sigma}_2)}{4d\norm{X}_{\psi_2}^2},\frac{(t^2 - d\norm{\Sigma}_2)^2}{64c_2\norm{X}_{\psi_2}^4}\right\}\bigg).
\end{align*}
\end{lemma}

\begin{proof}
By the (exponential) Markov inequality, we have
\BEQ\begin{split}
\Prob(\norm{X}_2 \ge t) = \Prob(\exp(\lambda\norm{X}_2^2) \ge \exp(\lambda t^2))\\ \le \frac{\Expect[\exp(\lambda\norm{X}_2^2)]}{\exp(\lambda t^2)}.\end{split}
\label{eq:large-deviation-sub-Gaussian-vector-1}
\EEQ
Let $Z$ be uniformly distributed on the unit sphere $\bS^{d-1}$. Then $\sqrt{d}Z$ is isotropic so $d\Expect[\langle Z,a\rangle^2] = \norm{a}_2^2$ for all $a$. We use this fact to bound the m.g.f. of $\norm{X}_2^2$:
\begin{align*}
\Expect[\exp(\lambda\norm{X}_2^2)] &= \Expect_X[\exp(\lambda d\Expect_{Z}[\langle Z,X\rangle^2])] \\
&\le  \Expect_X[\Expect_{Z}[\exp(\lambda d\langle Z,X\rangle^2)]].
\end{align*}
We interchange the order of expectation to obtain
\begin{align*}
\Expect[\exp(\lambda\norm{X}_2^2)] &\le \Expect_Z[\Expect_{X}[\exp(\lambda d\langle X,Z\rangle^2)]] \\
&\le \sup\,\{\Expect_{X}[\exp(\lambda d\langle X,Z\rangle^2)]\mid z\in\bS^{d-1}\} .
\end{align*}
$\langle X, z\rangle$ is sub-Gaussian (for fixed $z$) so $\langle X, z\rangle^2$ is (noncentered) sub-exponential. If $\lambda d < c/\norm{\langle X, z\rangle^2 - \Expect[\langle X, z\rangle^2]}_{\psi_1}$, then
\begin{align}
&\Expect[\exp(\lambda d\langle X, z\rangle^2)]\nonumber \\
&\pc\le \exp(\lambda d\Expect[\langle X, z\rangle^2])\Expect[\exp(\lambda d(\langle X, z\rangle^2 - \Expect[\langle X, z\rangle^2]))] \nonumber\\
&\pc\le \exp(\lambda d\Expect[\langle X, z\rangle^2] + c d^2\lambda^2\norm{\langle X, z\rangle^2 - \Expect[\langle X, z\rangle^2]}_{\psi_1}^2)\nonumber \\
&\pc\le \exp(\lambda d\Expect[\langle X, z\rangle^2] + 4c d^2\lambda^2\norm{\langle X, z\rangle^2}_{\psi_1}^2) \nonumber \\
&\pc\le \exp(\lambda d\Expect[\langle X, z\rangle^2] + 16c d^2\lambda^2\norm{\langle X, z\rangle}_{\psi_2}^4).
\label{eq:mgf-bound}
\end{align}
We substitute this bound into \eqref{eq:large-deviation-sub-Gaussian-vector-1} to obtain
$$
\Prob(\norm{X}_2 \ge t) \le \exp(16c d^2\lambda^2\norm{X}^4_{\psi_2} + \lambda (d\norm{\Sigma}_2 - t^2)),
$$
where $\Sigma$ is the covariance matrix of $X$. We optimize over $\lambda$ to obtain
$$
\Prob(\norm{X}_2 \ge t) \le \exp\left(-\frac{(t^2 - d\norm{\Sigma}_2)^2}{64c d^2\norm{X}_{\psi_2}^4}\right).
$$
If the optimum lies outside the region where the m.g.f.\ bound holds \eqref{eq:mgf-bound}, we can always choose 
$$
\lambda = \frac{c}{4d\norm{X}_{\psi_2}^2} \le \frac{c}{d\norm{\langle X, z\rangle^2 - \Expect[\langle X, z\rangle^2]}_{\psi_1}}
$$
in to obtain the tail bound:
$$
\Prob(\norm{X}_2 \ge t) \le \exp\left(c^3 + \frac{c(d\norm{\Sigma}_2 - t^2)}{4d\norm{X}_{\psi_2}^2}\right).
$$
We combine these two bounds to obtain
$$
\Prob(\norm{X}_2 \ge t) \le c_1\exp\left(-\min\left\{\frac{c_2(t^2 - d\norm{\Sigma}_2)}{4d\norm{X}_{\psi_2}^2},\frac{(t^2 - d\norm{\Sigma}_2)^2}{64c_2\norm{X}_{\psi_2}^4}\right\}\right).
$$
Note that the $t^2$ bound always holds. However, for small $t$, the $t^4$ term yields a tighter bound.
\end{proof}

\section{Proof of Lemma~1 (Weak Convergence of $\hat{r}$)}
\begin{proof}
We expand $\hat{r}_n - r$ (and neglect higher order terms) to obtain
\begin{align}
\norm{\hat{r}_n - r}_2 &= \Bigl\|\frac{1}{n}\sum_{i=1}^n\hat{\Sigma}^{-1} Y_i(X_i - \hat{\mu}) - \Sigma^{-1}\Expect[s(X)(X-\mu)]\Bigr\|_2 \\
&\le \norm{\hat{\Sigma}^{-1}}_2\Bigl\|\frac{1}{n}\sum_{i=1}^n Y_i(X_i - \hat{\mu}) - \Expect[s(X)(X-\mu)]\Bigr\|_2 \\
&\pc+\|\Expect[s(X)(X-\mu)]\|_2\|\hat{\Sigma}^{-1} - \Sigma^{-1}\|_2 + (o_P(1))^2.
\label{eq:r-1}
\end{align}
The higher order terms generically look like
\BEQ
\Prob(\abs{X - \Expect[X]}\abs{Y - \Expect[Y]} > \epsilon).
\label{eq:composite-terms}
\EEQ
We apply the union bound to deduce
\begin{align*}
&\Prob(\abs{X - \Expect[X]}\abs{Y - \Expect[Y]} > \epsilon) \\
&\pc\le \Prob(\abs{X - \Expect[X]} > \sqrt{\epsilon}) +  \Prob(\abs{X - \Expect[X]} > \sqrt{\epsilon}).
\end{align*}
For any $\epsilon <  1$, $\sqrt{\epsilon} > \epsilon$ and we have
$$
\Prob(\abs{X - \Expect[X]} > \sqrt{\epsilon}) \le \Prob(\abs{X - \Expect[X]} > \epsilon).
$$
Since terms of the form $\Prob(\abs{X - \Expect[X]} > \epsilon)$ appear in the upper bounds we derive, we can handle terms like \eqref{eq:composite-terms} with a constant factor (say 2). Since our bounds involve multiplicative constant factors anyways, we neglect these terms  to simplify our derivation.

We expand the first term to obtain
\begin{align*}
&\Bigl\|\frac{1}{n}\sum_{i=1}^n Y_i(X_i - \hat{\mu}) -\Expect[Y(X-\mu)]\Bigr\|_2 \\
&\pc\le |\Expect[s(X)]|\norm{\hat{\mu} - \mu}_2 + \norm{\mu}_2\Bigl|\frac{1}{n}\sum_{i=1}^n Y_i - \Expect[s(X)]\Bigr|_2 \\
&\pc\pc+ \Bigl\|\frac{1}{n}\sum_{i=1}^n Y_i(X_i - \mu) - \Expect[s(X)(X-\mu)]\Bigr\|_2 \\
&\pc\pc + (o_P(1))^2.
\end{align*}
$\hat{\mu} - \mu$ is a sub-Gaussian random variable with sub-Gaussian norm $\frac{\norm{X}_{\psi_2}}{\sqrt{n}}$, so there exist universal $c_1$ and $c_2$ s.t.\
$$
\Prob(\norm{\hat{\mu} - \mu}_2 > t) \le c_1\exp\left(-\frac{c_2n(t^2 - d\norm{\Sigma}_2)}{4d\norm{X}_{\psi_2}^2}\right).
$$ 
$Y$ is bounded between 1 and -1, so 
\BNUM
\item We can use Chernoff's inequality to deduce
$$
\Prob\biggl(\Bigl|\frac{1}{n}\sum_{i=1}^n Y_i - \Expect[s(X)]\Bigr| > t \bigr) \le 2\exp(-nt^2/2).
$$
\item $Y_i(X_i - \mu)$ are sub-Gaussian. Thus there exist universal $c_1$ and $c_2$ such that
\begin{align*}\begin{split}
\Prob\biggl(\Bigl\|\frac{1}{n}\sum_{i=1}^n Y_i(X_i - \mu) - \Expect[s(X)(X - \mu)]\Bigr\|_2 > t \bigr) \le\\ c_1\exp\left(-\frac{c_2n(t^2 - d\norm{\Sigma}_2)}{4d\norm{X}_{\psi_2}^2}\right).\end{split}
\end{align*}
\ENUM

We expand the second term in \eqref{eq:r-1} to obtain
\begin{align*}
\|\hat{\Sigma}^{-1} - \Sigma^{-1}\| = \|\Sigma^{-1/2}\|\|\Sigma^{1/2}\hat{\Sigma}^{-1}\Sigma^{1/2} - I\|\|\Sigma^{-1/2}\|.
\end{align*}
We expand the middle term to obtain
\begin{align*}
&\|\Sigma^{-1/2}\hat{\Sigma}\Sigma^{-1/2} - I\| 
&\pc\le \Bigl\|\Bigl(\frac1n\sum_{i=1}^n \Sigma^{-1/2}(X_i - \mu)(X_i - \mu)^T\Sigma^{-1/2}\Bigr)^{-1} - I\Bigr\|_2 \\
&\pc\pc+2\norm{\mu}_2\norm{\hat{\mu} - \mu}_2 + (o_P(1))^2
\end{align*}
We use Theorem 5.39 in \cite{vershynin2010introduction} to bound the first term:
$$
\Prob\biggl(\Bigl\|\frac1n\sum_{i=1}^n W_iW_i^T - I\Bigr\|_2 >  t\Bigr) \le 2\exp(-c_1'(\sqrt{n}t - c_2'\sqrt{d})^2),
$$
where $c_1',c_2'$ depend on the sub-Gaussian norm of $W$. We substitute these bounds into our expression for $\|\hat{r} - r\|_2$ to deduce
$$\textstyle
\Prob(\|\hat{r} - r\|_2 \ge \epsilon) \le C\exp\left(-\min\left\{\frac{c_1n\epsilon^2}{d},\bigl(c_1'\sqrt{n}\epsilon - c_2'\sqrt{d}\bigr)^2\right\}\right),
$$
where $C$ is an absolute constant and $c_1, c_1',c_2'$ depend on the sub-Gaussian norm of $X$. 
\end{proof}

\section{Proof of Lemma 3 (Weak Convergence of $\hat{Q}$)}

Let $\tilde{r}$ denote the projection of $r$ onto $\linspan\{u,v\}$.

\begin{lemma}
If $\|\hat{r} - r\|_2 \le \epsilon_0 = \min\,\{\alpha_1,\alpha_2\}\sin(\theta)$, then $\tilde{r}$ also lies in the positive quadrant. 
\end{lemma}

\begin{proof}
$r$ lies in interior of the conic hull of $\{u_1,\dots,u_k\}$, so we can express $r$ as $\sum_{i=1}^k\alpha_iu_i$, where $r_i > 0$. If $\tilde{r}$ also lies in the conic hull, then $\tilde{r} = \sum_{i=1}^k\beta_iu_i$ for some $\beta_i > 0$. 
\begin{align*}
\epsilon_0 &= \norm{\tilde{r} - r}_2 = \norm{\sum_{i=1}^k(\alpha_i - \beta_i)u_i}_2 = \sqrt{(\alpha - \beta)UU^T(\alpha - \beta)} \\
&\ge \norm{\alpha - \beta}_2 \sigma_{\min}(U) \ge \norm{\alpha - \beta}_\infty \sigma_{\min}(U).
\end{align*}
To ensure $\beta$ is component-wise positive, we must have $\norm{\alpha - \beta}_\infty < \min\{\alpha_1,\dots,\alpha_k\}$. A sufficient condition is
$\epsilon_0 \le \min\{\alpha_1,\dots,\alpha_k\}\sigma_{\min}(U)$.
\end{proof}

We are now ready to prove Lemma 3. We expand $\|\hat{Q}_n - Q\|_2$ (and neglect higher order terms) to obtain
\begin{align*}
\|\hat{Q}_n - Q\|_2 &\le \Bigl\|\frac1n\sum_{i=1}^nZ_i\Sigma^{-1/2}(X_i - \mu)(X_i - \mu)^T\Sigma^{-1/2} - Q\Bigr\|_2 \\
&\pc + 2\|\hat{\Sigma}^{-1/2} - \Sigma^{-1/2}\|_2\Expect\bigl[\|(X-\mu)(X-\mu)\Sigma^{-1/2}\|_2\bigr] \\
&\pc + 2\|\Sigma^{-1/2}\|_2\|\hat{\mu} - \mu\|_2\Expect\bigl[\|(X-\mu)\Sigma^{-1/2}\|_2\bigr] \\
&\pc + (o_P(1))^2.
\end{align*}
The second and third terms can be bounded using the same bounds used in the analysis of how fast $\hat{r}$ converges to $r$. Thus we focus on how fast 
$$
\sum_{i=1}^n Z_i\Sigma^{-1/2}(X_i - \mu)(X_i - \mu)^T\Sigma^{-1/2} 
$$ 
coverges to $Q$. First, we note that
\begin{align*}
&\Prob\biggl(\Bigl\|\frac1n\sum_{i=1}^n Z_iW_iW_i^T - Q\Bigr\|_2 > t \Bigr) \\
&\pc\le \Prob\biggl(\Bigl\|\frac1n\sum_{i=1}^n Z_iW_iW_i^T - Q\Bigr\|_2 > t\mid\hat{r}\in B_{\epsilon_0}(r) \Bigr) \\
&\pc\pc+ \Prob(\hat{r}\in B_{\epsilon_0}(r)).
\end{align*}
Let $\tilde{Z}_i$ denote the ``corrected'' version of the $Z_i$'s, \ie\ the $Z_i$'s we obtain if we use the projection of $\hat{r}$ onto the $\linspan\{u,v\}$ to flip the labels, and $W_i$ denote $\Sigma^{-1/2}(X_i-  \mu)$. We have
\begin{align}
&\Bigl\|\frac1n\sum_{i=1}^n Z_iW_iW_i^T - Q\Bigr\|_2 \\
&\pc\le \Bigl\|\frac1n\sum_{i=1}^n \tilde{Z}_iW_iW_i^T - Q\Bigr\|_2 + \Bigl\|\frac1n\sum_{i=1}^n (Z_i-\tilde{Z}_i)W_iW_i^T\Bigr\|_2.
\label{eq:corrected-z}
\end{align}
The probability the first term is large is bounded by
\begin{align*}
&\Prob\biggl(\Bigl\|\frac1n\sum_{i=1}^n \tilde{Z}_iW_iW_i^T - Q\Bigr\|_2 > t \mid \hat{r}\in B_{\epsilon_0}(r)\Bigr) \\
&\pc\le \sup_{\hat{r}\in B_{\epsilon_0}(r)}\Prob\biggl(\Bigl\|\frac1n\sum_{i=1}^n \tilde{Z}_iW_iW_i^T - Q\Bigr\|_2 > t\mid\hat{r}\in B_{\epsilon_0}(r)\Bigr)
\end{align*}
The $Z_i$'s are independent of $W_i$'s because the $Z_i$'s were computed using $\hat{r}$ that was in turn computed independently of the $X_i$'s. Thus this is a sum of \iid\ r.v. and we can bound it by
\begin{align*}
&\sup_{\hat{r}\in B_{\epsilon_r}(r)}\Prob\biggl(\Bigl\|\frac1n\sum_{i=1}^n \tilde{Z}_iW_iW_i^T - Q\Bigr\|_2 > t\mid\hat{r}\in B_{\epsilon_0}(r)\bigr) \\
&\pc\le 2\exp(-c_1(\sqrt{n}t - c_2\sqrt{d})^2),
\end{align*}
where $c_1,c_2$ depend on the sub-Gaussian norm of $\tilde{Z}W$. This is a consequence of Remark 5.40 in \cite{vershynin2010introduction}.

We now focus on bounding the second term in \eqref{eq:corrected-z}. Let $W_i$ denote the spherically symmetric (whitened) version of $X$. We restrict ourselves to the 3d subspace spanned by $u,v,\hat{r}$. Let $C_{\hat{r}}$ (for cone) be the ``bad'' region, \ie\ the region where $Z \ne \tilde{Z}$. We have
$$
\Bigl\|\frac1n\sum_{i=1}^n (Z_i-\tilde{Z}_i)W_iW_i^T\Bigr\|_2 =\Bigl\|\frac2n\sum_{i=1}^n \ones_{C_{\hat{r}}}(W_i)W_iW_i^T\Bigr\|_2.
$$
$\ones_{C_{\hat{r}}}$ is bounded, hence $\ones_{C_{\hat{r}}}(W_i)W_i$ is sub-Gaussian and
\begin{align*}
&\Prob\bigg(\Bigl\|\frac2n\sum_{i=1}^n \ones_{C_{\hat{r}}}(W_i)W_iW_i^T- 2\Expect[\ones_{C_{\hat{r}}}(W)WW^T] \Bigr\|_2> t\mid\hat{r}\in B_{\epsilon_0}(r)\big) \\
&\pc\le 2\exp(-c_1(\sqrt{n}t - c_2\sqrt{d})^2),
\end{align*}
where $c_1,c_2$ depend on the sub-Gaussian norm of $X$. It remains to bound $\norm{\Expect[\ones_{C_{\hat{r}}}(W)WW^T]}_2$. 

We use Jensen's inequality to obtain
\begin{align*}
&\pc\norm{\Expect[\ones_{C_{\hat{r}}}(W)WW^T]}_2 \le \Expect[\ones_{C_{\hat{r}}}(W)\norm{WW^T}_2] \\
&\pc\le C\Prob(W\in C_{\hat{r}}),
\end{align*}
where the constant $C$ depends on the number of mixture components.
Finally, we bound $\Prob(W\in C_{\hat{r}})$. 

The distribution of $W_i$ is spherically symmetric so the probability that $W_i\in C_{\hat{r}}$ is proportional to the surface area of the set
$$
S = \{w\in\bS^{k+1}\mid \hat{r}^Tw \le 0, u_1^Tw, \dots, u_k^Tw \ge 0\}.
$$
This set is contained in the set 
$$
\bar{S} = \{w\in \bS^{k+1}\mid \norm{w - u_1}_2 \le \norm{\hat{r} - r}_2\}
$$
This set is nothing but the spherical cap of radius $\norm{\hat{r} - r}_2$ in $k+1$ dimensions. An upper bound is 
$$
\gamma(\bar{S}) \le \exp(-(k+1)\cos(\frac{2-r^2}{2ab})^2),
$$
where $\gamma$ denotes the uniform measure on the unit sphere.
We substitute these bounds into our expression for the second term to obtain
\begin{align*}
&\pc\Expect[\ones_{C_{\hat{r}}}(W)WW^T] \\
&\pc\le C\arccos\left(\frac{2 - \epsilon_r^2}{2}\right)(1 - \cos\angle(u,v))(1 + \cos\angle(u,v) + \\
&\pc\pc\cos^2\angle(u,v)) \\
&\textstyle\pc\le C\arccos\left(1 - \frac{\epsilon_r^2}{2}\right).
\end{align*}
We combines these bounds to deduce
\begin{align*}
&\Prob(\|\hat{Q} - Q\|_2 > C\arccos\left(1 - \frac{\epsilon_r^2}{2}\right) + \epsilon\mid\hat{r}\in B_{\epsilon_0}(r)) \\
&\pc\le C\exp\left(-\min\left\{\frac{c_1n\epsilon^2}{d},\bigl(c_1'\sqrt{n}\epsilon - c_2'\sqrt{d}\bigr)^2\right\}\right).
\end{align*}
Finally, we have
\begin{align*}
&\Prob(\|\hat{Q} - Q\|_2 > \epsilon\mid\hat{r}\in B_{\epsilon_0}(r)) \\
&\pc\le \Prob(\|\hat{Q} - Q\|_2 > \epsilon \mid \norm{\hat{r} - r}_2 \le \sqrt{1 - \cos(\epsilon/2)}, \hat{r}\in B_{\epsilon_0}(r)) \\
&\pc\pc+ \Prob(\norm{\hat{r} - r}_2 \le \sqrt{1 - \cos(\epsilon/2)}).
\end{align*}
Let 
$$
f(\epsilon) := \min\left\{\frac{c_1n\epsilon^2}{d},\bigl(c_1'\sqrt{n}\epsilon - c_2'\sqrt{d}\bigr)^2\right\}.
$$
The first term is bounded by $C\exp\left(-f(\epsilon/2)\right)$ and the second term is bounded by $C\exp\left(-f(\sqrt{1 - \cos(\epsilon/2)})\right)$. We deduce
\begin{align*}
&\Prob(\|\hat{Q} - Q\|_2 > \epsilon\mid\norm{\hat{r} - r}_2 \le \epsilon_0) \\
&\pc\le C(\exp(-f(\sqrt{1-\cos(\epsilon/2)})) + \exp(-f(\epsilon/2)))\\
&\pc\le C\exp(-f(\min\{\sqrt{1-\cos(\epsilon/2)},\epsilon/2\})).
\end{align*}

\section{Principal Hessian Directions}

In this section, we apply the principal Hessian directions (pHd) \cite{li1992principal} method to our setting, and demontrate its failure to discover the space spanned by parameter profile when $\mu =0$. Recall that pHd considers a setting in which  features $X_i\in \reals^d$ are i.i.d.~and normally distributed with mean $\mu$ and covariance $\Sigma$, while labels $Y_i\in R$ lie, in expectation, in a $k$-dimensional manifold. In particular, some smooth $h:\reals^k\to\reals$, $k\ll d$:
$$\expect[Y\mid X= x] = h(\<u_1,x\>,\ldots,\<u_k,x\>)$$
where $u_\ell\in \reals^d$, $\ell \in [k]$. The method effectively creates an estimate 
$$\hat{H} =
n^{-1}\sum_{i=1}^nY_i\, \Sigma ^{-\frac12}X_iX_i^T\Sigma^{-\frac12}\in\reals^{d\times d}
$$
of the Hessian
\begin{align}\begin{split}H=\expect[\nabla_x^2 h(\<u_1,X\>,\ldots,\<u_k,X\>)] =\\ U^T \expect[\nabla^2_v h(\<u_1,X\>,\ldots,\<u_k,X\>)]  U,\end{split}\label{pHd:hessian} \end{align} where $\nabla_v^2h$ is the Hessian of the mapping $v\mapsto h(v)$, for $v\in \reals^k$, and $U$ the matrix of profiles. As in our case, the method discovers $\linspan(u_1,\ldots,u_k)$ from the eigenvectors of the eigenvalues that ``stand out'', after appropriate rotation in terms of $\Sigma$.

 Unfortunately, in the case of linear classifiers, 
$$h(v) =  \sum_{\ell=1}^kp_\ell g(v_\ell) $$ 
for $g(s) = 2f(s)-1$ is anti-symmetric. As a result, $\nabla^2_v h$ is a diagonal matrix whose $\ell$-th entry in the diagonal is $g''(v_\ell)$. Since $g$ is anti-symmetric, so is $g''$. Hence, if $\mu = 0$ we have that $\expect[g''(\langle u_\ell,X \rangle)]=0$; hence, the Hessian $H$ given by \eqref{pHd:hessian} will in fact be zero, and the method will fail to discover any signal pertaining to $\linspan(U)$.

This calls for an application of the pHd method to a transform of the labels $Y$. This ought to be non-linear, as an affine transform would preserve the above property. Moreover, given that these labels are binary,  polynomial transforms do not add any additional signal to $Y$, and are therefore not much help in accomplishing this task. In contrast, the `mirrorring' approach that we propose  provides a means of transforming the labels so that their expectation indeed carries sufficient information to extract the span $U$, as evidenced by Theorem~1.
\bibliography{references}
\bibliographystyle{icml2014}